\documentclass[sigconf]{acmart}

\usepackage{microtype}
\usepackage{textcomp}
\usepackage{graphicx}
\usepackage{subfigure}
\usepackage{booktabs} 
\usepackage{subcaption} 
\usepackage{hyperref}
\usepackage{algorithm}
\usepackage{algorithmic}
\usepackage{hyperref}
\usepackage{amsmath}
\usepackage{mathtools}
\usepackage{amsthm}
\usepackage[capitalize,noabbrev]{cleveref}
\theoremstyle{plain}
\newtheorem{theorem}{Theorem}

\newtheorem{definition}{Definition}

\usepackage{xcolor}
\newcommand{\rev}[1]{{#1}}
\AtBeginDocument{%
  }

\copyrightyear{2026}
\acmYear{2026}
\setcopyright{cc}
\setcctype{by}
\acmConference[MobiHoc '26]{The Twenty-seventh International Symposium on Theory, Algorithmic Foundations, and Protocol Design for Mobile Networks and Mobile Computing}{November 23--26, 2026}{Tokyo, Japan}
\acmBooktitle{The Twenty-seventh International Symposium on Theory, Algorithmic Foundations, and Protocol Design for Mobile Networks and Mobile Computing (MobiHoc '26), November 23--26, 2026, Tokyo, Japan}
\acmDOI{10.1145/3842721.3849662}
\acmISBN{979-8-4007-2975-1/2026/11}

\begin{document}

\title{Restless Bandits with Individual Penalty Constraints: Near-Optimal Indices and Deep Reinforcement Learning
}

\author{Nida Zamir}
 \email{nidazamir@tamu.edu}
\affiliation{%
 \institution{Texas A\&M University}
 \city{College Station}
 \state{Texas}
 \country{USA}}

\author{I-Hong Hou}
 \email{ihou@tamu.edu}
\affiliation{%
 \institution{Texas A\&M University}
 \city{College Station}
 \state{Texas}
 \country{USA}}

\renewcommand{\shortauthors}{Zamir and Hou}
\begin{abstract}

This paper investigates the Restless Multi-Armed Bandit (RMAB) framework under individual penalty constraints to address resource allocation challenges in dynamic wireless networked environments. Unlike conventional RMAB models, our model allows each user (arm) to have distinct and stringent performance constraints, such as energy limits, activation limits, or age of information minimums, enabling the capture of diverse objectives including fairness and efficiency. To find the optimal resource allocation policy, we propose a new Penalty-Optimal Whittle (POW) index policy. The POW index of an user only depends on the user's transition kernel and penalty constraints, and remains invariable to system-wide features such as the number of users present and the amount of resource available. This makes it computationally tractable to calculate the POW indices offline without any need for online adaptation. Moreover, we theoretically prove that the POW index policy is asymptotically optimal while satisfying all individual penalty constraints. We also introduce a deep reinforcement learning algorithm to efficiently learn the POW index on the fly. Simulation results across various applications and system configurations further demonstrate that the POW index policy not only has near-optimal performance but also significantly outperforms other existing policies.


\end{abstract}

\maketitle
 \section{Introduction}\label{sec:Intro} 

Efficient resource allocation is a fundamental challenge in dynamic networked systems where limited resources must be distributed among multiple users whose states evolve over time. A widely used framework for modeling such sequential resource allocation problems is the Restless Multi-Armed Bandit (RMAB), which has been applied to a broad range of network scheduling problems, including age-of-information minimization \cite{hsu2018age}, holding cost minimization \cite{borkar2017opportunistic}, scheduling real-time flows \cite{xu2019scheduling}, video streaming \cite{hosseini2017restless}, and user association \cite{singh2022user}.

Despite its broad applicability, a key limitation of classical RMAB formulations is that they focus solely on maximizing the total system reward. In many practical systems individual users often have their own service requirements. For example, a video streaming flow may require quality-of-service (QoS) guarantees, a battery-powered Internet-of-Things (IoT) device may operate under strict energy constraints, and patients in healthcare applications may require regular monitoring or treatment. In such scenarios, an effective resource allocation policy must not only maximize the overall system performance but also satisfy user-specific requirements. While several studies have explored constrained RMAB formulations, existing approaches either focus on specific types of requirements \cite{wang2024online, lodi2024fairness, joseph2016fairness} or lack strong performance guarantees \cite{bura2024windex}.

In this paper, our goal is to propose a new resource allocation algorithm that maximizes the total reward, satisfies all individual service requirements, and is computationally tractable and learnable. We first extend the RMAB model to incorporate the service requirements of users. Specifically, we consider that a user generates both a reward and a \emph{penalty} in each time step based on its state and the controller's allocation decision. Each user requires the long-term accumulated penalty to be less than some upper-bound. Our \rev{problem formulation} allows any arbitrary choice of the penalty function and is therefore able to capture virtually all kinds of service requirements.

Our RMAB model with individual penalty constraints is computationally infeasible to solve directly due to the curse-of-dimensionality. Motivated by the celebrated Whittle index policy \cite{whittle1988restless}, which is asymptotically optimal for the RMAB problem under mild conditions, we seek to find an index policy that is asymptotically optimal for our RMAB model with individual penalty constraints. Similar to the Whittle index policy, our approach is also based on studying the dual problem. However, a key challenge arises for our problem since the dual problem of our formulation involves two sets of Lagrange multipliers, one for the constraint on resource availability and the other for the penalty constraints. This makes it difficult to express the dual problem with only one scalar. To address this challenge, we show that there is a coupling between the two sets of Lagrange multipliers. We then leverage this coupling to reformulate the dual problem into one with only one Lagrange multiplier. This observation allows us to define a new index, which we call the \emph{Penalty-Optimal Whittle} (POW) index, for each user. We further propose a computationally tractable approach for calculating the POW indices.

To understand the performance of the POW index policy, we analyze its behavior at the fluid limit when both the number of users and the amount of resource scale proportionally. We theoretically prove that, under the POW index policy, the total reward matches a theoretical performance upper-bound and the penalty of each user satisfies its penalty constraint. Hence, the POW index policy is asymptotically optimal. 

In addition to the setting where the system dynamics are known, we also consider a more practical scenario in which the transition dynamics and application behaviors of users are unknown to the controller. We develop a deep reinforcement learning algorithm, DeepPOW, that learns the POW indices directly from interaction with the environment.

The performance of both the POW index policy and the DeepPOW are extensively evaluated in simulations. Simulation results based on three practical network scheduling problems show that the POW index policy not only achieves near-optimal performance but also significantly outperforms other baseline policies. They also show that the DeepPOW learns the POW index efficiently.

In summary, the main contributions of this paper are as follows:
\begin{enumerate}
    \item We propose a new framework of restless bandit with individual penalty constraint. (Section~\ref{sec:systemmodel}).
    \item We develop a Penalty-Optimal Whittle (POW) index policy for the problem of restless bandit with individual penalty constraints and show that the POW index policy is asymptotically optimal (Sections~\ref{sec:pow_index} and \ref{sec:fluid}).
    \item We propose an online deep reinforcement learning algorithm called DeepPOW that can learn the POW index on the fly without any prior knowledge of the application behaviors (Section~\ref{sec:policy_gradient}).
    \item We conduct comprehensive simulations to evaluate the performance of both the POW index policy and the DeepPOW (Section~\ref{sec:simulation}).
\end{enumerate}

 \section{Related work}\label{sec:relatedwork}

The Whittle index~\cite{whittle1988restless} is the dominant approach for RMABs. Weber and Weiss proved its asymptotic optimality under a global attractor condition~\cite{weber1990index}, and Gast et al.\ established exponential convergence rates~\cite{gast2023exponential}. Larra\~{n}aga, Ayesta, and Verloop established asymptotic optimality in multi-class queueing systems~\cite{larranaga2015asymptotically}. Computing and testing the Whittle index has also been studied~\cite{gast2023testing}. Learning the Whittle index from data without system knowledge has attracted growing interest; representative approaches include NeurWIN~\cite{nakhleh2021neurwin}, DeepTOP~\cite{nakhleh2022deeptop}, tabular methods~\cite{robledo2024tabular}, and optimistic index learning~\cite{wang2023optimistic}. 
None of these works consider individual penalty constraints.

Several works incorporate per-arm individual requirements into RMAB models. Kadota et al.~\cite{kadota2018optimizing} study AoI minimization with per-device throughput constraints. Tang et al.~\cite{tang2020minimizing} address uplink AoI minimization with per-user power constraints, and Chen et al.~\cite{chen2023minimizing} study AoI minimization under peak power constraints. Fairness constraints in RMAB settings have been considered in~\cite{li2023avoiding,herlihy2023planning}. Wang et al.~\cite{wang2024online} study a fairness-constrained RMAB with a reinforcement learning approach, though without scalability guarantees. Li and Varakantham~\cite{li2022efficient} propose FaWT for fair resource allocation. Bura et al.~\cite{bura2024windex} address resource allocation with throughput and time-since-last-service constraints in NextG networks, but without asymptotic performance guarantees. Each of these works captures only a specific type of individual constraint; none provides a general index policy framework with provable optimality for arbitrary per-arm penalty constraints.

\section{System Model and Problem Formulation}\label{sec:systemmodel} 

{In this section, we explore the problem of RMAB with individual constraint and its application to various networked systems. Throughout this paper, we use $\vec{x}$ to denote the vector containing $[x_1, x_2, \dots]$. }

We consider a system composed of $N$ restless arms, indexed by $n = 1, 2, \dots, N$. In each time step $t$, a controller observes the state of each arm $n$, denoted by $s_{n,t}\in S_n$, and then chooses $C$ arms to activate. We use $a_{n,t}$ to be the indicator function that the controller activates arm $n$ at time step $t$. Based on $s_{n,t}$ and $a_{n,t}$, the arm $n$ will generate a global reward $r_{n,t}$ and an individual penalty $g_{n,t}$. We assume that the reward, penalty, and state evolution of each arm follows a Markov decision process (MDP). Specifically, when an arm $n$ is in state $s_n$ and an action $a_n$ is taken, then it generates a random reward with mean $r_n(s_n, a_n)$, a random penalty with mean $g_n(s_n, a_n)$ and changes its state to state $s'_n$ with probability $P_n(s'_n|s_n,a_n)$.

Assume that each arm $n$ has a service constraint on its incurred penalty. The goal of the controller is to maximize the expected discounted total reward while satisfying the penalty constraint of each individual arm. Specifically, let $\beta$ be the discount factor and let $B_n$ be the penalty constraint of arm $n$. The controller aims to maximize $E[\sum_t\sum_n \beta^tr_{n,t}]$ under the constraint $E[\sum_t \beta^tg_{n,t}]\leq \sum_t\beta^tB_n=B_n/(1-\beta)$, for all $n$.

 Let $\vec{s}_t:=[s_{1,t}, s_{2,t},\dots, s_{N,t}]$ and $\vec{a}_t:=[a_{1,t}, a_{2,t},\dots,a_{N,t}]$, then the controller's policy can be viewed as a function $\vec{\pi}$ that determines $\vec{a}_t=\vec{\pi}(\vec{s}_t):=[\pi_{1}(\vec{s}_t), \pi_{2}(\vec{s}_t),\dots]$. The controller's goal is to find the optimal $\vec{\pi}$ for the following optimization problem:
 
\begin{align}
    &\textbf{\textbf{SYSTEM:}} \nonumber\\
     & \begin{aligned}
     \max_{\vec{\pi}} & E[\sum_{t=0}^\infty \sum_{n=1}^N \beta^tr_n(s_{n,t}, \pi_{n}(\vec{s}_t))] \\
    \text{s.t. } &\sum_{t=0}^\infty  E[\beta^tg_n(s_{n,t},\pi_{n}(\vec{s}_t))] \leq \frac{B_n}{1-\beta}, \forall n, \\
   & \sum_{n=1}^N \pi_{n}(\vec{s}_t)\leq C, \forall t, \\
    \text{and } & \pi_n(\vec{s}_t)\in\{0,1\}, \forall{n}.
    \end{aligned} \label{equation:SYSTEM4}
\end{align}

Our model of RMAB with individual constraints naturally extends several recent work on RMAB with fairness constraints in terms of activation rates \cite{wang2024online, herlihy2023planning, li2022efficient, mao2024time}, which is the special case when $g_{n,t}=1-a_{n,t}$ under our model. In addition, our model applies to many other real-world scenarios. For example, in wireless scheduling a BS serves $N$ data flows subject to per-flow throughput or energy constraints, where each flow's state captures channel quality and application metrics such as AoI or queue length. In online advertising, a website displays $C$ ads at a time, each with a per-advertiser budget constraint on display frequency, and the goal is to maximize total click-through revenue.

\section{A New Index Policy}\label{sec:pow_index}

The \textbf{SYSTEM} problem is computationally infeasible to solve directly due to the state space $\vec{s}$ being the Cartesian product of the state spaces of each $s_n$, causing its size to grow exponentially with $N$. When there is no penalty constraint, which is equivalent to the special case when $B_n=\infty$, the traditional Whittle index policy leverages the Lagrange decomposition technique to develop a near-optimal policy. However, as the \textbf{SYSTEM} problem has two sets of constraints, and therefore two sets of Lagrange multipliers, the Whittle index policy cannot be easily applied to it. In this section, we will introduce a new index policy, called the Penalty-Optimal Whittle (POW) index, that is able to use one single scalar for each arm to capture both its reward preference and its penalty requirement. In addition to its simplicity, the POW index of an arm depends only on the MDP and the penalty constraint of the arm itself. It is independent of any other system configuration, such as the capacity $C$ and the MDPs of other arms, making it particularly suitable for dynamic systems where arms can join/leave frequently. This feature is in sharp contrast to several recently proposed index policies, such as the Lagrangian index \cite{avrachenkov2024lagrangian}, the Net Gain index \cite{shisher2023learning}, and the Partial index \cite{zou2021minimizing}, all of which can only be computed after knowing the complete system configurations.
 
\subsection{Lagrange Decomposition}\label{lagDecomp}
 
Our approach utilizes the Lagrange decomposition that transforms the high-dimensional \textbf{SYSTEM} problem into several low-dimensional subproblems.  First, we relax the per-step capacity constraint to an average constraint in order to apply the Lagrangian decomposition
\begin{align}
    &\textbf{\textbf{SYS-Relaxed:}}&\nonumber \\
    & \begin{aligned}
    \max_{\vec{\pi}} & E[\sum_{t=0}^\infty \sum_{n=1}^N \beta^tr_n(s_{n,t},\pi_{n}(\vec{s}_t))] \\
    \text{s.t. } &\sum_{t=0}^\infty  E[\beta^tg_n(s_{n,t}, \pi_{n}(\vec{s}_t))] \leq \frac{B_n}{1-\beta}, \forall n, \\
   & E[\sum_{t=0}^\infty\sum_{n=1}^N\beta^t \pi_{n}(\vec{s}_t)]\leq\sum_{t=0}^\infty \beta^t C, \\
    \text{and } & \pi_n(\vec{s}_t)\in\{0,1\}, \forall {t, n}.
    \end{aligned}\label{equation:SYSTEM1}
\end{align}

Now, we introduce the Lagrange multipliers $\lambda$ and $\vec{\mu}$ for the capacity constraint and the penalty constraint of an arm $n$, respectively. The Lagrangian of the \textbf{SYS-Relaxed} problem is as follows:
\begin{align}
 &{\textbf{SYS-Lagrangian:}}&\nonumber \\
  & \begin{aligned}
 L(\lambda,\vec{\mu})&:=\max_{\vec{\pi}} \sum_{t=0}^\infty \sum_{n=1}^NE\big[ \beta^t(r_n(s_{n,t}, \pi_{n}(\vec{s}_t))-\lambda \pi_{n}(\vec{s}_t)\\
 &-\mu_ng_{n,t}(s_{n,t}, \pi_{n}(\vec{s}_t)))\big] + \sum_n\mu_n\frac{B_n}{1-\beta}+\lambda\frac{C}{1-\beta}
\end{aligned}\label{equation:SYS-Relaxed-Lagragian}
\end{align}
Due to its summation form, the \textbf{SYS-Lagrangian} problem can be decomposed into $N$ sub-problems.\rev{With a slight abuse of notation, we use $\pi_n(s_{n,t})$
to denote the policy for arm $n$, which depends only on its
local state. The \textbf{Arm-n} problem is defined as follows:}
\begin{align} 
&{\textbf{Arm-n:}}&\nonumber \\
&\begin{aligned}
L_n(\lambda,{\mu_n}) &:=\max_{{\pi_n}} \sum_{t=0}^\infty E\big[ \beta^t(r_n(s_{n,t}, \pi_{n}(s_{n,t}))-\lambda \pi_{n}(s_{n,t})\\
& -\mu_ng_{n,t}(s_{n,t}, \pi_{n}(s_{n,t})))\big] + \mu_n\frac{B_n}{1-\beta} 
 \end{aligned}\label{Arm_n_lagrangian}
\end{align}
We ignore the \(\lambda\frac{C}{1-\beta}\) term in the \textbf{Arm-n} problem, as it is a constant. Let \(\pi_n^*(s_n, \lambda, \mu_n)\) be the optimal solution to the \textbf{Arm-n} problem, then $\pi_1^*(s_n, \lambda, \mu_n), \pi_2^*(s_n, \lambda, \mu_n),\dots$ together solve the \textbf{SYS-Lagrangian} problem.

So far, our approach is similar to the standard Lagrange decomposition approach used in standard RMAB problems for the derivation of the Whittle Indices. However, there is a key difference: In standard RMAB problems, the \textbf{Arm-n} problem can be expressed as a function of one single variable, $L_n(\lambda)$, because, as there are no penalty constraints, $\mu_n$ does not exist in standard RMAB problems. It is indeed the one-dimensional representation of the \textbf{Arm-n} problem that makes it feasible to define the Whittle index as a scalar. In contrast, the \textbf{Arm-n} problem in our formulation needs to incorporate both $\lambda$ and $\mu_n$ to capture both the capacity constraint and the penalty constraint.

To address the challenge posed by the two-dimensional representation of the \textbf{Arm-n} problem, our key insight is to establish a coupling between $\lambda$ and $\mu_n$, instead of treating them as two independent parameters. Given $\lambda$, we define the \textbf{Dual-Arm-n} problem as one to find the optimal $\mu$:
 \begin{align}
 \begin{aligned}
{\textbf{Dual-Arm-n:}} \mbox{ } \min_{{\mu_n \geq 0}}\mbox{ } & L_n(\lambda,{\mu_n}). 
 \end{aligned}\label{Arm_n_dual}
 \end{align}

\rev{Let $\mathcal{M}_n^*(\lambda) := \operatorname*{arg\,min}_{\mu_n\geq 0} L_n(\lambda,\mu_n)$ denote the set of optimal solutions to the Dual-Arm-\(n\) problem. To resolve possible non-uniqueness, we select $\mu_n^*(\lambda):=\min\mathcal{M}_n^*(\lambda)$.} We then define the \textbf{SYS-Dual} problem as follows:
 \begin{align}
 \begin{aligned}
&{\textbf{SYS-Dual:}} \mbox{ } \min_{\lambda} \sum _ {n=1}^{N} L_n(\lambda,{\mu_n^*(\lambda)})+\lambda\frac{C}{1-\beta}, 
 \end{aligned}\label{equation:SYS-dual}
 \end{align}
and let $\lambda^*$ be the optimal solution to the \textbf{SYS-Dual} problem. By doing so, $\lambda^*$ implicitly addresses both the capacity constraint and all penalty constraints of all arms.

\subsection{The Penalty-Optimal Whittle Index Policy}\label{sec:POW} 

 One significant drawback of the Lagrange decomposition approach is its reliance on relaxing the per-step capacity constraint, transforming it into an average constraint. In this section, we establish the POW index policy that satisfies the per-step capacity constraint without relaxation.
 

Recall that $\mu_n^*(\lambda)$ is the optimal solution to the \textbf{Dual-Arm-n} problem under a given $\lambda$. Also, $\pi_n^*(\cdot)$ is the optimal solution to the \textbf{Arm-n} problem under given $\lambda$ and $\mu$. Hence,  for a given $\lambda$, if we choose $\mu_n=\mu_n^*(\lambda)$, then the optimal strategy for the \textbf{Arm-n} problem will activate the arm at state $s_n$ if and only if $\pi_n^*(s_n, \lambda, \mu_n^*(\lambda))=1$. We then define the POW index as follows:

\begin{definition}\label{index} [Penalty-Optimal Whittle (POW) index]
For a given arm \( n \), the \emph{Penalty-Optimal Whittle index} for each state \( s_n \) of arm \( n \), denoted by \( I_{n}(s_n) \), is defined as:
\begin{align}
I_{n}(s_n) := \sup \left\{ \lambda \mid \pi_n^*(s_n, \lambda, \mu_n^*(\lambda)) = 1 \right\}.
\end{align}
\end{definition}

The POW index can be interpreted as the maximum cost that an arm $n$ is willing to pay for activation when it is in the state $s_n$. Intuitively, the controller will activate the arm $n$ as long as $\lambda \leq I_n(s_n)$. Arm $n$ is said to be \emph{indexable} when it indeed exhibits such a behavior:
\begin{definition}\label{indexability} [Indexability] An arm $n$ is indexable if, for any $s_n$ and any $\lambda'\leq I_{n}(s_n)$, $\pi_n^*(s_n, \lambda', \mu_n^*(\lambda'))=1$.
\end{definition}

After finding the POW indices, the controller will simply activate the $C$ arms with the highest POW indices in each time step. We call this policy the \emph{POW index policy}.


\subsection{Calculating the POW Index and Checking for Indexability} \label{sec:calculate_index}

We now discuss how to calculate the POW indices and evaluate whether an arm is indexable. 

We note that the \textbf{Arm-n} problem is equivalent to a MDP whose reward is $r_n(s_{n,t}, a_{n,t})-\lambda a_{n,t}-\mu_ng_{n,t}(s_{n,t}, a_{n,t})$. For given $\lambda$ and $\mu_n$, let $v_{n,s_n}$ be the value function of arm $n$ at state $s_n$, then we have the Bellman equation


\begin{equation}
\begin{aligned}
    v_{n,s_n} &= \max_{a_n}  r_n(s_n,a_n)- \lambda a_n-\mu_n g_n(s_n,a_n)\\ &+\beta\sum_{s'_n}P(s'_n|s_n,a_n)v_{n,s'_n}.
\end{aligned}
\end{equation}
For a given $\lambda$, we can find both $\mu_n^*(\lambda')$ and the corresponding $v_{n,s_n}$ by solving the following simple linear equation, where $\alpha_n(s_n)$ is the initial state distribution. 
\begin{equation}
\begin{aligned}
 &\mbox{\textbf{LP-Dual-Arm-n:}} \\
 &\min_{\mu_n \geq 0,v_{n,s_n}} \hspace{0pt} \sum_{s_n}\alpha_n(s_n)v_{n,s_n} + \mu_n\frac{B_n}{1-\beta}   \\
    &  \hspace{6pt}\mbox{s.t. }\hspace{16pt} v_{n,s_n} \geq  r_n(s_n,a_n)- \lambda a_n-\mu_n g_n(s_n,a_n) \\
    & \hspace{70pt}+\beta\sum_{s'_n}P(s'_n|s_n,a_n)v_{n,s'_n}, \forall (s_n,a_n). \label{equation:LP-Dual-Arm-n}
\end{aligned}
\end{equation}

After finding the value function $v_{n,s}$ for a given $\lambda$ and the corresponding $\mu^*(\lambda)$, we construct the optimal policy 
$\pi_n^*(s_n, \lambda, \mu_n^*(\lambda))$ by setting $\pi_n^*(s_n, \lambda, \mu_n^*(\lambda))=1$ if the following condition holds:
\begin{equation}
\begin{aligned}
v_{n, s_n} &= r_n(s_n, 1) - \lambda - \mu_n^*(\lambda) g_n(s_n, 1) \\
& + \beta \sum_{s_n'} P(s_n' \mid s_n, 1) v_{n, s_n'},\nonumber
\end{aligned}
\end{equation}
and $\pi_n^*(s_n, \lambda, \mu_n^*(\lambda))=0$ otherwise.

After constructing $\pi_n^*(s_n, \lambda, \mu_n^*(\lambda))$, finding the POW index requires only a line-search over $\lambda$, which also checks the indexability of each arm.

In practice, there may be some states whose POW Indices are $+\infty$ or $-\infty$. This can happen when the individual penalty constraint is so tight that, in order to satisfy the penalty constraint, the restless arm is required to take a certain action in some states regardless of the value of $\lambda$. To avoid such extreme situations, we impose an additional constraint $\mu_n\leq U$, for some sufficiently large $U$, when solving \textbf{LP-Dual-Arm-n}.



\section{Asymptotic Optimality of the POW Index Policy}\label{sec:fluid} 

In this section, we study the asymptotic performance of our POW index policy at the fluid limit. We first introduce the fluid limit. Given a system with $N$ restless arms and a scalar $K$, we construct an alternative system as follows: First, for each arm $n$ in the original system, there are $K$ arms in the alternative system with the same transition kernel and penalty constraint as arm $n$. Hence, there are $KN$ restless arms in the alternative system, with arms $1, 2, \dots, K$ having the same transition kernel as arm 1 in the original system, arms $K+1, K+2, \dots, 2K$ having the same transition kernel as arm 2 in the original system, etc. To simplify the notation, we divide the $KN$ arms into $N$ groups and say that group $n$ consists of arms $(n-1)K+1, (n-1)K+2, \dots, nK$. Second, the capacity of the alternative system is $KC$.  The long-term average discounted reward of this alternative system is $E[\sum_{t=0}^\infty \sum_{m=1}^{KN} \beta^tr_m(s_{m,t}, a_{m,t})]$. The fluid limit considers the behavior of this alternative system when $K\rightarrow \infty$ and we define the expected discounted reward at the fluid limit as $\lim_{K\rightarrow\infty}E[\sum_{t=0}^\infty \sum_{m=1}^{KN} \beta^tr_m(s_{m,t}, a_{m,t})]/K$.

We now discuss the policy of the system at the fluid limit. Let $y_{n,s,t}$ be the portion of group-$n$ arms that are in state $s$ at time $t$. We call $\vec{y}_t:=[y_{n,s,t}]$ the state distribution of the fluid limit at time $t$. A policy $\pi$ of the fluid limit system observes the state distribution $\vec{y}_t$ and then determines which arms to activate. Let $z_{n,s,a, t}$ be the portion of group-$n$ arms that are in state $s$ and take action $a$ under the policy at time $t$ and let $\vec{z}_t:=[z_{n,s,a,t}]$. Thus, a policy can be described as a mapping from $\vec{y}_t$ to $\vec{z}_t$ and we say that $\pi(\vec{y}_t)=\vec{z}_t$. Since the action taken by each arm is either 0 or 1, we require $z_{n,s,0,t}+z_{n,s,1,t}=y_{n,s,t}$.

Next, we discuss the reward and the state transition of the fluid limit system under an arbitrary policy $\pi$. Recall that when a group-$n$ arm is in state $s$ and takes action $a$, its expected one-step reward is $r_n(s,a)$, its expected one-step penalty is $g_n(s,a)$, and it will transition to state $s'$ with probability $P_n(s'|s,a)$. Since there are $Kz_{n,s,a,t}$ group-$n$ arms that are in state $s$ and take action $a$ under $\pi$, the strong law of large numbers dictates that the total reward and total penalty obtained by these $Kz_{n,s,a,t}$ arms are \rev{$Kz_{n,s,a,t}r_n(s,a)+o(K)$ and $Kz_{n,s,a,t}g_n(s,a)+o(K)$}, respectively, and the number of arms among them that transition to state $s'$ is \rev{$Kz_{n,s,a,t}P_n(s'|s,a)+o(K)$}. Thus, as $K\rightarrow\infty$, we have $\sum_{m=1}^{KN} r_m(s_{m,t}, a_{m,t})/K\rightarrow\sum_n\sum_s\sum_ar_n(s,a)z_{n,s,a,t}$ and $y_{n,s,t+1}\rightarrow\sum_{s'}\sum_{a'}P_n(s|s',a')z_{n,s',a',t}$ almost surely. 

We start by considering the policy for the relaxed problem as shown in Eq. (\ref{equation:SYSTEM1}). Let $x_{n,s,a}:=\sum_{t=0}^{\infty}\beta^t z_{n,s,a,t}$. Then we have, at the fluid limit, 
\begin{align}
  &\lim_{K\rightarrow\infty}\frac{E[\sum_{t=0}^\infty \sum_{m=1}^{KN} \beta^tr_m(s_{m,t}, a_{m,t})]}{K}
  =&\sum_n\sum_s\sum_ar_n(s,a)x_{n,s,a},
\end{align}
and
\begin{align}
  &\lim_{K\rightarrow\infty}\frac{E[\sum_{t=0}^\infty \sum_{m=K(n-1)+1}^{Kn}\beta^tg_m(s_{n,t}, a_{n,t})]}{K}\nonumber\\
  =&\sum_s\sum_ag_n(s,a)x_{n,s,a}, \forall n.
\end{align} 
Since  $z_{n,s,0,t}+z_{n,s,1,t}=y_{n,s,t}$, we have $ \sum_a x_{n,s,a}=\sum_{t=0}^\infty \beta^ty_{n,s,t}$. Moreover, since $y_{n,s,t+1}$ converges to $\sum_{s'}\sum_{a'}P_n(s|s',a')z_{n,s,a,t}$ at the fluid limit, we have $\sum_{t=0}^\infty \beta^ty_{n,s,t+1}=\sum_{s'}\sum_{a'}P_n(s|s',a')x_{n,s',a'}$. Therefore, we have 
\begin{equation}
    y_{n,s,0}= \sum_a x_{n,s,a}-\beta\sum_{s'}\sum_{a'}P_n(s|s',a')x_{n,s',a'}.
\end{equation} 

Let $\pi^{Rel}$ be the optimal policy for the relaxed problem at the fluid limit. Effectively, $\pi^{Rel}$ solves the following optimization problem given the initial state distribution $\vec{y}_0$:


\begin{align}
    & \textbf{Fluid-Relaxed($\vec{y}_0$):} \nonumber\\
    & \begin{aligned}
        \max_{\vec{x}} & \sum_n\sum_s\sum_a r_n(s,a)x_{n,s,a} \\
        \text{s.t.} & \sum_s\sum_a g_n(s,a)x_{n,s,a} \leq \frac{B_n}{1-\beta},  \forall n, \\
        & \sum_n\sum_s x_{n,s,1} \leq \frac{C}{1-\beta}, \\
        & \sum_a x_{n,s,a} - \beta \sum_{s'} \sum_{a'} P_n(s|s',a') x_{n,s',a'} = y_{n,s,0},  \forall n, s, \\
        \text{and } & x_{n,s,a} \geq 0, \forall n, s, a.
    \end{aligned} \label{equation:fluid-relaxed}
\end{align}

We now turn our attention to the POW index policy and study its performance at the fluid limit. Unlike $\pi^{Rel}$, which only needs to satisfy the relaxed capacity constraint, the POW index policy, which we denote by $\pi^{Ind}$ in this section, needs to ensure the capacity constraint, $\sum_n\sum_sz_{n,s,1,t}\leq C$, is satisfied in each time step. The POW index policy schedules the $KC$ arms with the highest POW indices in each time step. Hence, $\pi^{Ind}$ effectively solves the following optimization problem given the state distribution $\vec{y}_t$ at time $t$.

\begin{align}
    \mbox{\textbf{Fluid-Ind($\vec{y}_t$):}}& \nonumber \\ 
    \max_{\vec{z}_t}\mbox{ } & \sum_{n=1}^N\sum_s I_n(s)z_{n,s,1,t}\\
    \mbox{s.t. }    & \sum_n\sum_sz_{n,s,1,t}\leq C, \label{equation:fluid-index1}\\
   &z_{n,s,0,t}+z_{n,s,1,t}=y_{n,s,t},\forall n,s, \label{equation:fluid-index2}\\
    \mbox{and } &  z_{n,s,a,t}\geq 0, \forall n, s,a.\label{equation:fluid-index}
\end{align}

We note that since $\pi^{Rel}$ is designed for the relaxed problem, the total discounted reward of $\pi^{Rel}$ is naturally an upper-bound to the original unrelaxed problem. On the other hand, the index policy $\pi^{Ind}$ needs to satisfy the capacity constraint on a per-step basis. 

In the sequel, we will show that $\pi^{Rel}$ and $\pi^{Ind}$ are equivalent in steady state under two mild conditions that are commonly assumed in the analysis of index policies. As a result, the reward obtained by the POW index policy is the reward upper-bound and hence it is optimal at the fluid limit. We first formally define the two conditions.

\begin{definition} [Global attractor]
    A policy $\pi$ for the fluid limit system is said to have a \emph{global attractor} $\vec{y}^{\pi}=[y^{\pi}_{n,s}]$ if, for any initial state distribution $\vec{y}_0$, the state distribution converges to $\vec{y}^{\pi}$, that is, $\lim_{t\rightarrow\infty}\vec{y}_t=\vec{y}^{\pi}$.
\end{definition}

Intuitively, the global attractor condition states that the policy has a unique steady state.

\begin{definition} [Strict index separator] \label{D2}
    A state distribution $\vec{y}=[y_{n,s}]$ for the fluid limit system is said to have a \emph{strict index separator} $\lambda>0$ if $\sum_{n,s:I_n(s)>\lambda}y_{n,s}=C$.
\end{definition}

If a state distribution $\vec{y}_t$ has a strict index separator $\lambda$, then the index policy chooses $z_{n,s,1,t}=y_{n,s,t}$ if $I_n(s)>\lambda$, and $z_{n,s,1,t}=0$, otherwise. The strict index separator condition is therefore used to ensure that the solution to \textbf{Fluid-Ind($\vec{y}_t$)} is unique. \rev{Thus, at the fluid fixed point, the POW policy makes deterministic decisions without requiring randomization. Our optimality result is restricted to systems satisfying this separator condition.}
We now establish the following theorem:
\rev{\begin{theorem} \label{theorem:fluidopt} 
 If \(\pi^{\mathrm{Rel}}\) exists and has a global attractor \(\vec y^{\mathrm{Rel}}\), \(\pi^{\mathrm{Ind}}\) has a global attractor \(\vec y^{\mathrm{Ind}}\), \(\vec y^{\mathrm{Ind}}\) has a strict index separator \(\lambda^{\mathrm{Ind}}\), and, for every arm \(n\), Dual-Arm-\(n\) and its corresponding primal problem are feasible with finite optimal values at \(\lambda=\lambda^{\mathrm{Ind}}\), then \(\vec y^{\mathrm{Rel}}=\vec y^{\mathrm{Ind}}\) and \(\pi^{\mathrm{Rel}}(\vec y^{\mathrm{Ind}})=\pi^{\mathrm{Ind}}(\vec y^{\mathrm{Ind}})\).
\end{theorem}}

\begin{proof}
    The proof proceeds by showing that $(\vec{x}^{\text{Ind}}, \lambda^{\text{Ind}}, [\mu_n^*(\lambda^{\text{Ind}})])$ satisfies the KKT conditions of the \textbf{Fluid-Relaxed($\vec{y}^{\text{Rel}}$)} problem, where $\lambda^{\text{Ind}}$ is the Lagrange multiplier of the capacity constraint in \textbf{Fluid-Ind} and $\mu_n^*(\lambda^{\text{Ind}})$ is the optimal solution of \textbf{Dual-Arm-n} for that $\lambda^{\text{Ind}}$. Checking primal feasibility, dual feasibility, and complementary slackness follows the same steps as the analysis of the Whittle index policy~\cite{weber1990index}. The key difference lies in verifying Lagrangian optimality, which here involves two sets of Lagrange multipliers $\lambda$ and $\vec{\mu}$. Using the strict index separator condition and the definition of the POW index, one can show that $\pi^{\text{Ind}}$ activates arm $n$ in state $s_n$ if and only if $\pi_n^*(s_n, \lambda^{\text{Ind}}, \mu_n^*(\lambda^{\text{Ind}}))=1$, which is precisely the condition for Lagrangian optimality. Please see Appendix~\ref{sec:A} for the full proof.
\end{proof}

Theorem~\ref{theorem:fluidopt} shows that the POW index policy schedules the same set of users as the $\pi^{Rel}$ policy at the fluid limit and in the steady state. Since $\pi^{Rel}$ is one that achieves the largest rewards while satisfying all individual penalty constraints, the POW index policy is also able to do so at the fluid limit. Hence, the POW index policy is asymptotically optimal.

 \section{A Deep Reinforcement Learning Algorithm for Finding the POW Index}\label{sec:policy_gradient} 

In many practical scenarios, the transition dynamics of arms are unknown or difficult to model accurately in advance. In these cases, the POW indices need to be learned during runtime. In this section, we propose a deep reinforcement learning based index policy, called the \emph{Deep Penalty-Optimal Whittle (DeepPOW)} that learns the POW indices directly from data without requiring prior knowledge of the transition kernels.


\subsection{Problem Formulation and Gradients}
We parameterize the POW index and the dual variable using function approximators and formulate their learning as optimization problems. Specifically, we introduce a parameterized function $w^{\phi_n}_{n}(s_n)$ to approximate the POW index $I_n(s_n)$ and another parameterized function $u^{\zeta_n}_{n}(\lambda)$ to approximate the optimal dual variable $\mu_n^*(\lambda)$. A key difficulty is that neither the POW index nor the optimal dual variable can be directly observed from the system. To overcome this issue, we derive objective functions whose maximizer/minimizer correspond to the desired index and dual variable. We then derive the policy gradients with respect to these objective functions.


We first consider the \textbf{Arm-n} problem as shown in Eq. (\ref{Arm_n_lagrangian}). Given a $\lambda$ and $\mu_n=u_n^{\zeta_n}(\lambda)$, the \textbf{Arm-n} problem is equivalent to a MDP whose reward is $ r_n(s_n,a_n)- \lambda a_n-u^{\zeta_n}_n(\lambda)\big(g_{n}(s_{n}, a_{n})-{B_n}\big)$. Given $\phi_n$, we consider a policy for this MDP that chooses $a_{n,t}=1$ if  $w^{\phi_n}_{n}(s_{n,t})\geq \lambda$, and $a_{n,t}=0$, otherwise. Let $Q_{n}^{\phi_n}(s_n,a_n, \lambda, u_n^{ \zeta_n}(\lambda))$ be the state-action function of this policy.
\begin{align}
Q^{\phi_n}_{n}(s_n,a_n, \lambda,& u_n^{ \zeta_n}(\lambda))= r_n(s_n,a_n) - \lambda a_n-u^{\zeta_n}_n(\lambda)\big(g_{n}(s_{n}, a_n)-{B_n}\big)\nonumber \\ &+ \beta \sum_{s'_n}P(s'_n|s_n,a_n)Q^{\phi_n}_{n}(s'_n,a_n, \lambda, u_n^{ \zeta_n}(\lambda)),\label{equation:qfunction}
\end{align}

where $\mathbb{I}(\cdot)$ is the indicator function.

Recall that the POW index $I_n(s_n)$ is defined as the largest $\lambda$ such that the $\pi_n^*(s_n, \lambda, \mu^*_n(\lambda))=1$. Hence, when $u_n^{\zeta_n}(\lambda)=\mu_n^*(\lambda)$, setting $w^{\phi_n}_{n}(s_{n})=I_n(s_n)$ maximizes $Q^{\phi_n}_{n}(s_n,a_n, \lambda, u_n^{ \zeta_n}(\lambda))$ for any $s_n$ and any $\lambda$. Hence, we define the objective function for learning $\phi_{n}$ as 
\begin{align} \label{eq:jFunction}
J_n^{{\phi_{n}}} := \sum_{s \in S} \int_{\lambda=-M}^{\lambda=+M}  Q^{\phi_n}_{n}(s_n,a_n,  \lambda, u_n^{ \zeta_n}(\lambda))  d\lambda,
\end{align}
where $M$ is a sufficiently large constant such that $I_n(s_n) \in [-M,+M]$. Eq. (\ref{eq:jFunction}) transforms the problem of learning the unknown $I_n(s_n)$ into one of maximizing a well-defined objective function. Moreover, we show that there is a simple expression to the gradient of $J_n^{\phi_n}$.

\begin{theorem}\label{thm:policygrad} 
Given the parameter vectors ${\phi}_n$ and $\zeta_n$, if all states $s_n \in S_n$ have distinct values of $ w^{\phi_n}_{n}(s_n)$, then the gradient of the objective function $J_n^{{\phi_{n}}}$ with respect to the parameter vector ${\phi}_{n}$ is given by:
   \begin{align}\label{eq:jGradient} 
    \nabla _{\phi_{n}}J_n^{{\phi_{n}}}  = &  \sum_{s_n \in S_n}\left[Q^{\phi_n}_{n}(s_n,1, w^{\phi_n}_{n}(s_n), u_n^{ \zeta_n}( w^{\phi_n}_{n}(s_n))) \right. \nonumber \\ & \left.  -Q^{\phi_n}_{n}(s_n,0, w^{\phi_n}_{n}(s_n), u_n^{ \zeta_n}( w^{\phi_n}_{n}(s_n)))\right]\nabla _{\phi_{n}}  w^{\phi_n}_{n}(s_n).
\end{align}
\end{theorem}
\begin{proof}
The proof is similar to the one provided by Nakhleh et al.~\cite{nakhleh2022deeptop}. Please see the detailed proof in Appendix~\ref{sec:C}.
\end{proof}

Next, we derive the objective for learning $u_n^{ \zeta_n}(\lambda)$ to predict the optimal dual variable $\mu_n$. Consider the \textbf{Dual-Arm-$n$} problem, which can be expressed as

\begin{equation}
    \min_{\zeta_n}E[\max_{a_n}Q^{\phi_n}_{n}(s_n,a_n,  \lambda, u_n^{ \zeta_n}(\lambda))],
\end{equation} 
where the expectation is taken over the initial state distribution.

When $w_n^{\phi_n}(s_n)=I_n(s_n)$, setting $u_n^{\zeta_n}(\lambda)=\mu_n^*(\lambda)$ minimizes the above objective for any fixed $\lambda$. Therefore, we define the objective function for learning $\zeta_n$ as

\begin{align} \label{eq:kFunction}
K_n^{{\zeta_{n}}} =  \Big(\max_{a_n}Q^{\phi_n}_{n}(s_n,a_n,  \lambda, u_n^{ \zeta_n}(\lambda))\Big).
\end{align}
Assuming $u_n^{ \zeta_n}(\lambda)$ is twice differentiable with respect to $\zeta_n$ , then $K_n^{{\zeta_{n}}}$ is continuous and twice differentiable almost everywhere. The gradient of Eq. (\ref{eq:kFunction}) with respect to $\zeta_n$ is then given by

\begin{align} \label{eq:grad_kFunction}
\nabla_{\zeta_n} K_n^{{\zeta_{n}}} =\nabla_{\zeta_n}  \Big(\max_{a_n}Q^{\phi_n}_{n}(s_n,a_n,  \lambda, u_n^{ \zeta_n}(\lambda))\Big).
\end{align}

  \subsection{Deep Penalty-Optimal Whittle}\label{sec:Deeppow_index_policy} 
 
    We now introduce the Deep Penalty-Optimal Whittle (DeepPOW), a model-free, actor-critic deep reinforcement learning algorithm that learns the POW indices. DeepPOW employs two actor networks, parameterized by \( \phi_{n} \) and \( \zeta_{n} \), alongside a critic network parameterized by \( \theta_{n} \). The actor network \( \phi_{n} \) takes the state \( s_n \) as input and outputs a value \( w^{\phi_{n}}_{n}(s_n) \). The actor network \( \zeta_{n} \) takes the $\lambda$ as input and produces $u_n^{\zeta_n}(\lambda )$. DeepPOW aims to train $\phi_{n}$ and $\zeta_n$ such that $w^{\phi_{n}}_{n}(s_n)\approx I_{n}(s_n)$, $u_n^{\zeta_n}(\lambda )\approx \mu_n^*(\lambda)$ and  $ Z^{\theta_n}_{n}(s_n,a_n, \lambda, \mu_n^{\zeta}(\lambda)) \approx Q^{\phi_n}_{n}(s_n,a_n, \lambda, u_n^{ \zeta_n}(\lambda))$. The critic network $\theta_n$ takes $(s_n, a_n, \lambda,\mu_n)$ as input and produces a number $ Z^{\theta_n}_{n}(s_n,a_n, \lambda, \mu_n)$ that approximate the ${Q}^{\theta_n}_{n}(s_n,a_n, \lambda, \mu_n)$. DeepPOW aims to train $\theta_n$ so that it satisfies the Bellman equation
\begin{align}
 Z^{\theta_n}_{n}(s_n,a_n, \lambda,\mu_n)=& r_n(s_n,a_n)- \lambda a_n-\mu_n\big(g_{n}(s_{n}, a_{n})-B_n\big)\nonumber \\ &+ \beta\sum_{s'_n}P(s'_n|s_n,a_n)\max_{a'_n}Z^{\theta_n}_{n}(s'_n,a'_n,\lambda,\mu_n).\label{equation:qfunction-arm_n}
\end{align}
For more stable learning, DeepPOW also maintains a target critic network $\theta'_n$ that is updated slower than $\theta_n$.

We now discuss how DeepPOW activates arms and updates all neural networks. In each time step $t$, the controller sorts all restless arms by their predicted indices $w_n^{\phi_n}(s_{n,t})$. With probability $1-\epsilon$, the controller activates the $C$ arms with the highest $w_n^{\phi_n}(s_{n,t})$, and, with probability $\epsilon$, it activates $C$ randomly selected arms. The controller observes the reward $r_{n,t}$, the penalty $g_{n,t}$, and the state transition $s_{n,t+1}$. The controller then stores $(s_{n,t},a_{n,t},r_{n,t},g_{n,t},s_{n,t+1})$ in the replay buffer.

In each time step $t$, DeepPOW randomly samples a batch of transitions $(s_{n,t},a_{n,t},r_{n,t},g_{n,t},s_{n,t+1})$ from the replay buffer to update the neural networks. DeepPOW utilizes Theorem~\ref{thm:policygrad} to update the actor network $\phi_n$ and calculates the average of  
\begin{align}\label{eq:actor_gradient}
& \left[Z^{\theta_n}_{n}(s_{n,t},1, w_n^{\phi_n}(s_{n,t}), u_n^{\zeta_n}( w_n^{\phi_n}(s_{n,t})))\right.\nonumber \\
& \left.-Z^{\theta_n}_{n}(s_{n,t},0, w_n^{\phi_n}(s_{n,t}), u_n^{\zeta_n}( w_n^{\phi_n}(s_{n,t})))\right]\nabla _{\phi_{n}}  w^{\phi_n}_{n}(s_{n,t}),
\end{align}
over all transitions in the batch as an empirical estimate of $\nabla _{\phi_{n}}J_n^{{\phi_{n}}}$ and then uses it to update $\phi_{n}$.

DeepPOW updates the actor network $\zeta_n$ using Eq. (\ref{eq:grad_kFunction}). For each transition in the batch, DeepPOW randomly samples a value $\lambda \in [-M,M]$ and appends it to the transition. DeepPOW calculates the average of

\begin{align}\label{eq:grad_mu_update}
\nabla_{\zeta_n}
\Big[
\max_{a_n}
Z^{\theta_n}_{n}
\big(
s_{n,t},a_n,\lambda,u_n^{\zeta_n}(\lambda)
\big)
\Big].
\end{align}
over all transitions in the batch and uses it to update $\zeta_n$. To ensure $u_n^{\zeta_n}(\lambda)\in[0, U]$, we include a sigmoid function with proper scaling in the last layer of the actor network $\zeta_n$.


Finally, we discuss the update of the critic network $\theta_n$. DeepPOW samples ${\lambda}\in[-M,+M]$ and ${\mu_n}\in[0,U]$ for each transition in the batch. DeepPOW estimates the loss function as the average of 
\begin{align}\label{eq:q_lossFunction}
  &\left(Z^{\theta_n}_{n}(s_{n,t},a_{n,t}, \lambda, \mu_n)-r_{n,t} + {\lambda}a_{n,t} + \mu_n\big(g_{n,t}-B_n\big) \right. \nonumber \\
 &\left. -\beta \max_{a'_n}Z_{n}^{\theta'_n}(s_{n,t+1},a'_n, \lambda, \mu_n) \right)^2
\end{align}
over all transitions in the batch. The estimated gradient of the loss function is then the average of
\begin{align}\label{eq:q_gradlossFunction}
 &\left(Z^{\theta_n}_{n}(s_{n,t},a_{n,t}, \lambda, \mu_n)-r_{n,t} + {\lambda}a_{n,t} + \mu_n\big(g_{n,t}-B_n\big) \right. \nonumber \\
 &\left. -\beta \max_{a'_n}Z_{n}^{\theta'_n}(s_{n,t+1},a'_n, \lambda, \mu_n) \right) \nabla _{\theta_n}Z^{\theta_n}_{n}(s_{n,t},a_{n,t}, \lambda, \mu_n).
\end{align}
DeepPOW uses this estimated gradient to update $\theta_n$ and then applies a soft update to the target critic network $\theta'_n$. \rev{DeepPOW converges to a local optimum with appropriate step sizes and neural network settings.}

The complete algorithm is described in Alg.~\ref{alg:DeepPOW Index Policy}.

 \begin{algorithm}[tb]
        \caption{DeepPOW} \label{alg:DeepPOW Index Policy}
        \begin{algorithmic}
           \STATE Initialize $N$ actor networks with parameter vector $\phi_n$, $N$ actor networks with parameter vector $\zeta_n$, $N$ critic networks with parameter vector $\theta_n$, and $N$ target critic with parameter $\theta'_n$. Set $\theta'_n$ to $\theta_n$.

            \FOR{each time step $t$}
            \FOR{each arm $n$}
            \STATE Observe state $s_{n,t}$.
            \STATE Set index $I_{n,t}\leftarrow w_n^{\phi_n}(s_{n,t})$.
            \ENDFOR
            \STATE With probability $1-\epsilon$, activate the $C$ arms with the highest $I_{n,t}$.
            \STATE Otherwise, activate $C$ randomly chosen arms.
            \STATE Store $(s_{n,t}, a_{n,t}, r_{n,t}, g_{n,t}, s_{n,t+1})$ in the replay buffer of arm $n$.
          \FOR{n= 1, 2, \dots, N}
          \STATE Sample a batch of transitions $(s_{n,{t_k}},a_{n,{t_k}},r_{n,{t_k}},g_{n,{t_k}},s_{n,t_{k+1}})$ from the replay buffer of arm $n$.
          \STATE Append randomly selected $\lambda \in [-M, +M]$ and $\mu_n  \in [0, +U]$ to each transition in the batch.
          \STATE Set $\Delta\phi_n$ to be the average of Eq.~(\ref{eq:actor_gradient}) of all transitions.
          \STATE Set $\Delta\zeta_n$ to be the average of Eq.~(\ref{eq:grad_mu_update}) of all transitions.
          \STATE Set $\Delta\theta_n$ to be the average of Eq.~(\ref{eq:q_gradlossFunction}) of all transitions.
          \STATE Update $\phi_{n}$ by $\Delta \phi_{n}$, $\zeta_{n}$ by $\Delta\zeta_{n}$ and $\theta_n$ by $\Delta\theta_n$.
          \STATE $\theta'_n\leftarrow\tau \theta_{n}+(1-\tau)\theta'_n$.
          \ENDFOR
          \ENDFOR
          
        \end{algorithmic}
      \end{algorithm}

\section{Simulations}\label{sec:simulation}

In this section, we evaluate the proposed {POW index policy} and {DeepPOW} across three constrained RMAB applications under different network settings. We evaluate the performance of the POW index policy in the non-asymptotic regime, examining how quickly it converges to the fluid-limit benchmark and how it compares with several existing scheduling policies. We also evaluate the ability of DeepPOW to learn effective index policies in unknown environments by comparing it with POW and other learning-based baselines.

We use the LP-based method in Section~\ref{sec:calculate_index} to calulate the POW indices and check for indexability. In all three applications and all settings, we find that all arms are indexable.

\subsection{Simulation Overview}

We compare the POW index policy with: (1)~\textbf{Whittle index policy}, which ignores penalty constraints; (2)~\textbf{FaWT}~\cite{li2022efficient}, which restricts scheduling to constraint-satisfying arms with the highest Whittle indices; and (3)~\textbf{DPP}~\cite{neely2013lyapunov}, which uses Lyapunov optimization with virtual queue updates. We also report the \textbf{Fluid Relaxed} solution of~(\ref{equation:fluid-relaxed}) as an upper benchmark.

For DeepPOW, we compare with \textbf{DeepTOP}~\cite{nakhleh2022deeptop}, \rev{which does not account for individual penalty constraints},  and \textbf{DeepTOP+FaWT} (DeepTOP indices combined with FaWT arm selection).

To study finite-system convergence, we generate scaled systems with $KN$ arms and capacity $KC$ for increasing $K$. Results are averaged over 20 independent runs, each with 10,000 evaluation time steps. We report average total reward and average constraint violation with standard deviation bars. 

See Appendix~\ref{sec:impl} for training, implementation details, and parameter settings for all three scenarios.

\subsection{Throughput Maximization with Activation Constraint}



We consider a throughput maximization with activation constraint scenario motivated by IoT and 5G systems, where a BS schedules uplink transmissions for a set of user equipments (UEs) with time-varying channel qualities. In each slot, the BS observes the channel states of all UEs and decides which ones to schedule. When a UE is scheduled, it delivers a number of packets that depends on its channel state. Meanwhile, each UE has a minimum activation requirement: it must be scheduled at least a fraction $\delta_n$ of the time. The BS aims to maximize long-term throughput of all UEs while ensuring their activation requirements are met.

We now discuss how to model each UE as a restless arm with a penalty constraint. We denote the state of UE $n$ at time $t$ by $s_{n,t}$, taking values in the discrete set $\{1,2,\dots,50\}$. When UE $n$ is scheduled at state $s_n$, it delivers $\theta_{n,0}+\theta_{n,1}s_{n,t}$ packets, where $\theta_{n,0}$ and $\theta_{n,1}$ capture the baseline throughput and channel sensitivity of UE $n$, respectively. We assume that, for all $n$, $P_n(s_{n,t+1}=s' \mid s_{n,t}=s,a_{n,t}) = \frac{1}{50}$, for all $s,s' \in \{1,\dots,50\}$ and $a_{n,t}\in\{0,1\}$. The reward of UE $n$ is therefore defined as: $R_n(s_{n,t},a_{n,t}) = (\theta_{n,0}+\theta_{n,1}s_{n,t})a_{n,t}$. To enforce activation regularity, we define the penalty as: $g_n(s_{n,t},a_{n,t}) = 1-a_{n,t}$, with $B_n = 1-\delta_n$, where $\delta_n$ is the minimum activation fraction required by UE $n$. The parameter settings are summarized in Appendix~\ref{sec:impl}.

Fig.~\ref{fig:TP_reward} shows the average reward and average constraint violation under finite-system scaling. In both settings, the POW index policy remains consistently close to the fluid-relaxed benchmark. The rewards of the POW index policy are virtually identical to the fluid-relaxed benchmark. When $K$ is small, there is a small amount of constraint violation, but the violation approaches zero as $K$ increases. Compared with other baseline policies, POW consistently offers a superior reward–violation tradeoff. The Whittle index policy achieves the highest throughput but incurs high constraint violations, as it ignores penalty requirements. FaWT improves feasibility by restricting scheduling to constraint-satisfying arms, but this conservative selection reduces throughput. DPP maintains feasibility across all settings, yet its reward is lower because queue-based updates prioritize short-term correction of constraint violations over long-term scheduling efficiency. Overall, these results demonstrate that POW effectively balances throughput maximization and service regularity, achieving near-fluid performance while preserving constraint satisfaction in finite systems. \rev{This experiment assesses how closely finite systems approximate the fluid-limit performance characterized in Section~5.}

Fig.~\ref{fig:TP_deep} presents the learning performance under unknown transition dynamics for $(N,C,K)=(4,1,10)$ and $(6,1,10)$. DeepPOW converges to the reward level of the POW index policy while maintaining near-zero constraint violation, confirming successful recovery of the constrained index structure through interaction with the environment. The small residual violation observed for $(6,1,10)$ arises from averaging over 20 independent runs. While individual runs often achieve near-zero violation at most time steps, stochastic variability across runs slightly increases the mean. The standard deviation shown in Fig.~\ref{fig:TP_deep_b} highlights that constraint violations are typically negligible in single runs. Simulation results also show that DeepPOW outperforms both DeepTOP and DeepTOP+FaWT in terms of reward and feasibility. 

\begin{figure}[t]
    \centering
    \subfigure[Comparison of index policies with $N=4, C=1$]{
        \includegraphics[width=1\linewidth]{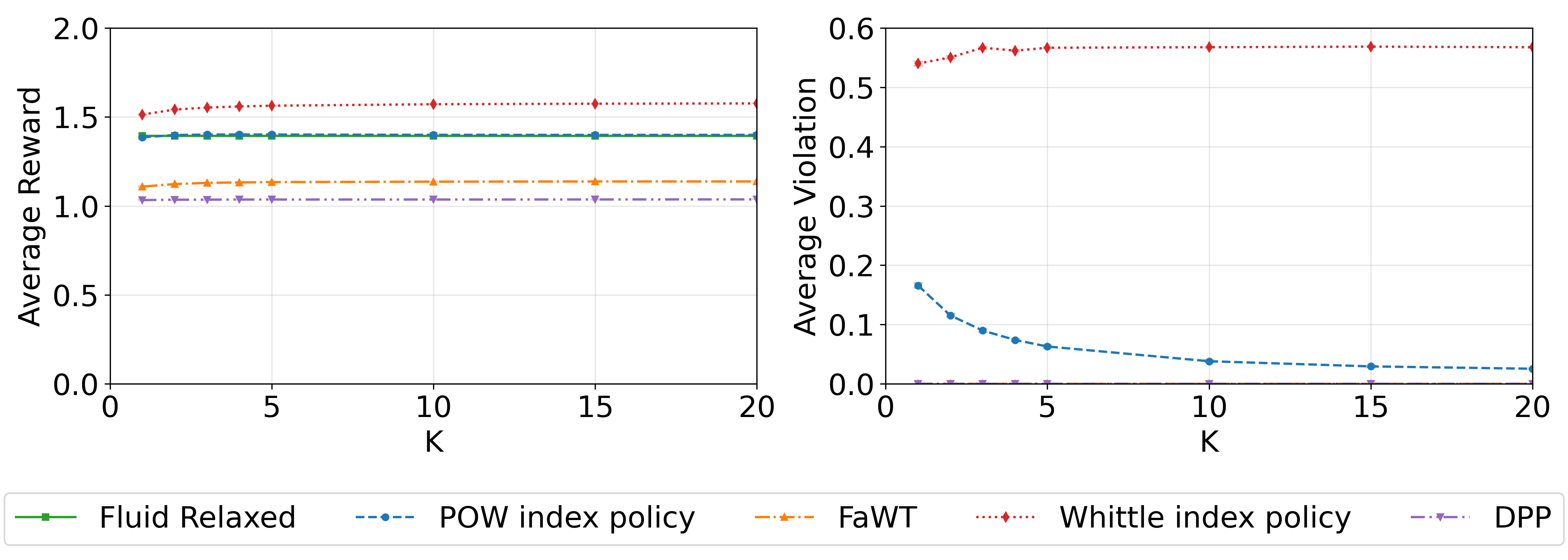}
    }

    \subfigure[Comparison of index policies with $N=6$ and $C=1$]{
        \includegraphics[width=1\linewidth]{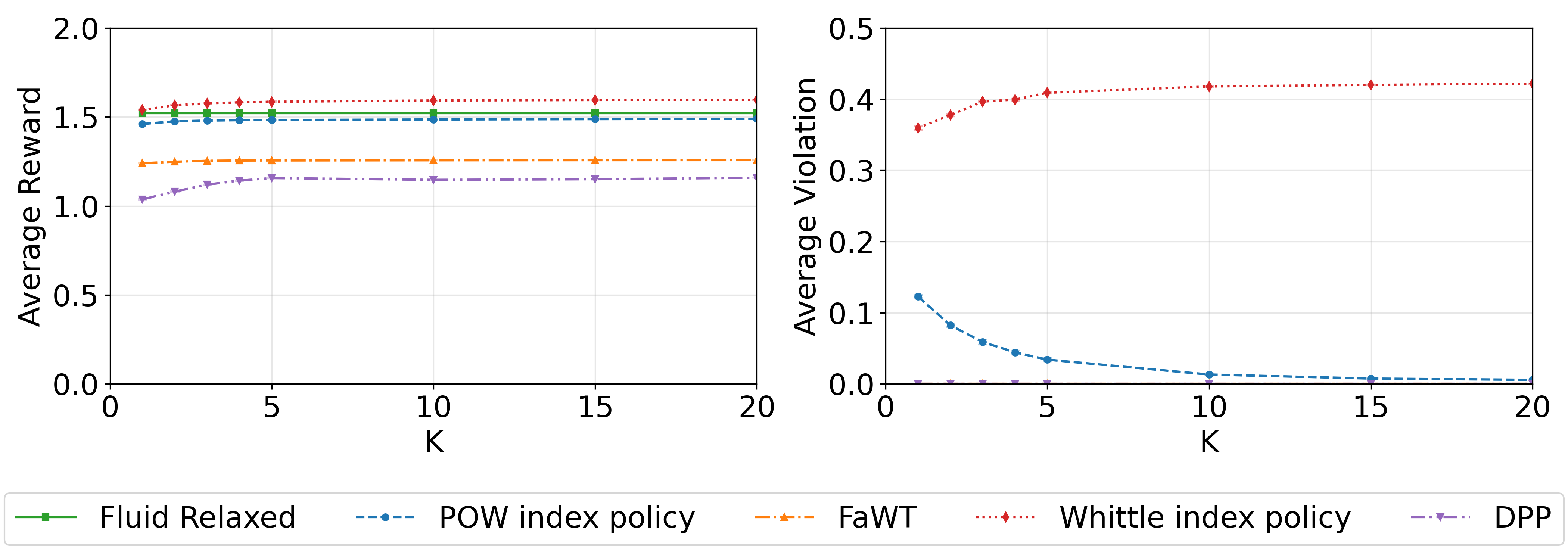}
    }
    \caption{Average reward and average constraint violation for throughput maximization with activation constraints.}
    \label{fig:TP_reward}
\end{figure}

\begin{figure}[t]
    \centering
    \subfigure[$N=4$ and $C=1, K=10$]{
        \includegraphics[width=1\linewidth]{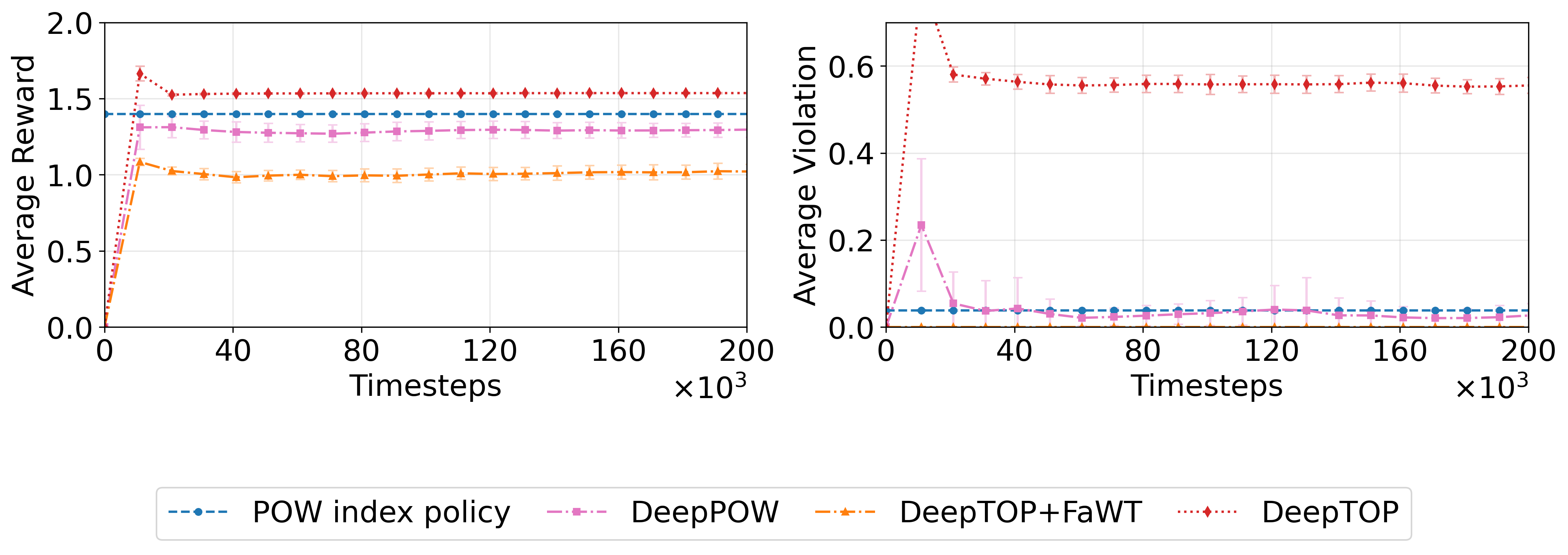}
    \label{fig:TP_deep_a}}

    \subfigure[$N=6, C=1$ and $K=10$]{
        \includegraphics[width=1\linewidth]{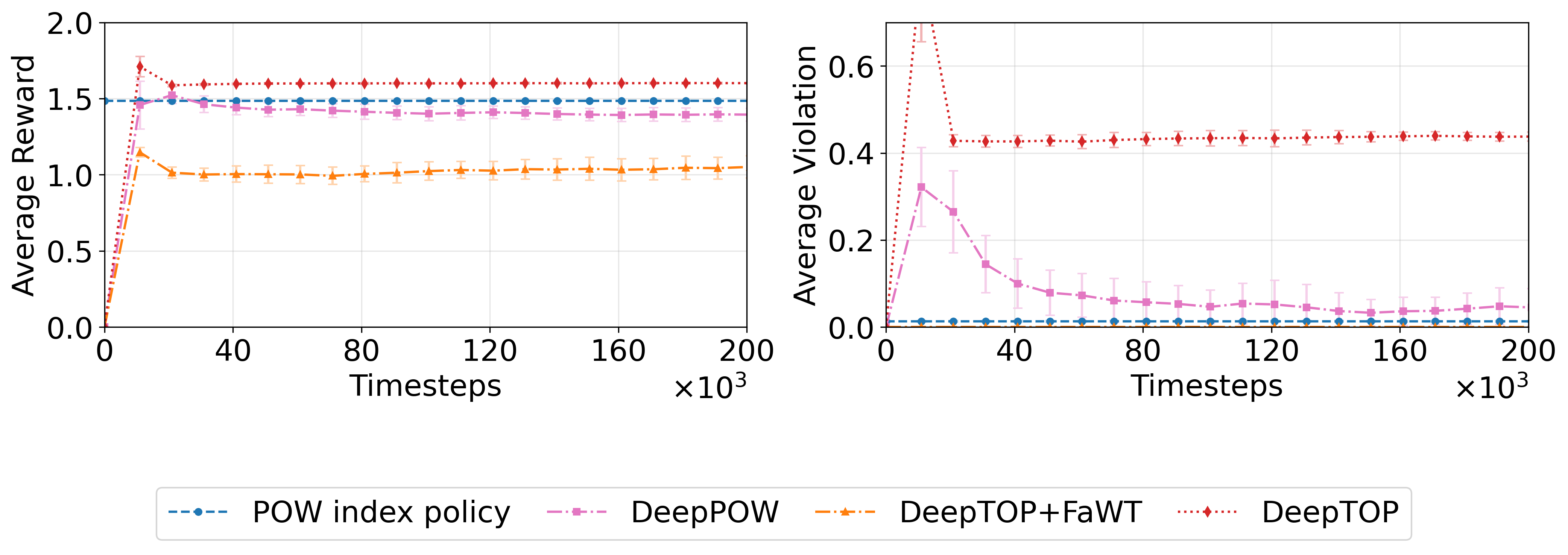}
    \label{fig:TP_deep_b}}
    \caption{Comparison of learning policies for throughput maximization with activation constraints.}
    \label{fig:TP_deep}
\end{figure}

\subsection{Remote Sensing}

\begin{figure}[t]
    \centering
    \subfigure[Comparison of index policies with $N=4$ and $C=1$]{
        \includegraphics[width=1\linewidth]{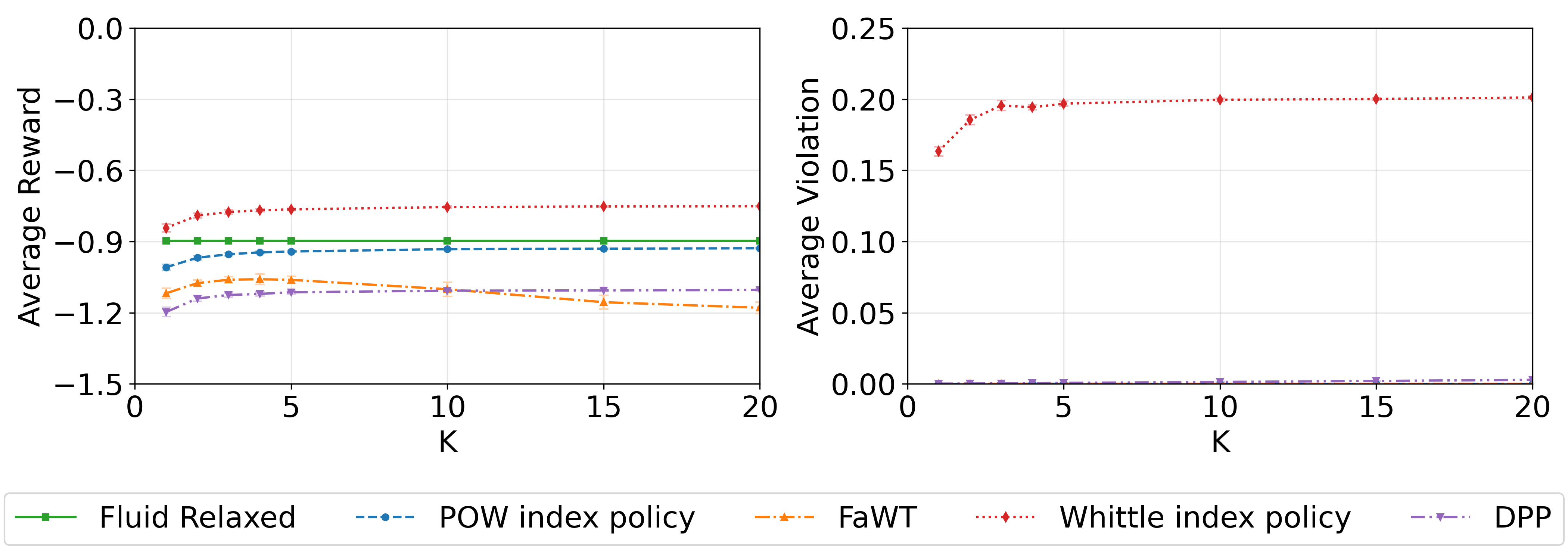}
    }

    \subfigure[Comparison of index policies with $N=8$ and $C=1$]{
        \includegraphics[width=1\linewidth]{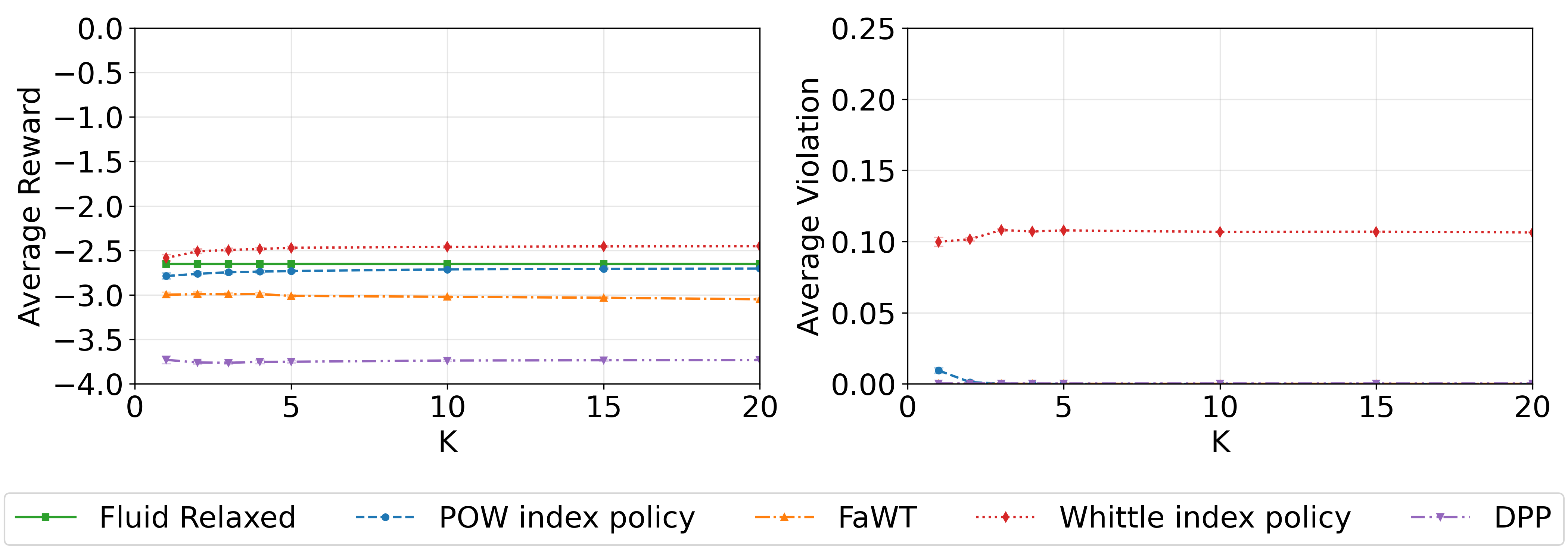}
    }
    \caption{Average reward and average constraint violation for remote sensing.}
    \label{fig:RS_reward}
\end{figure}

\begin{figure}[t]
    \centering
    \subfigure[$N=4$ and $C=1$]{
        \includegraphics[width=1\linewidth]{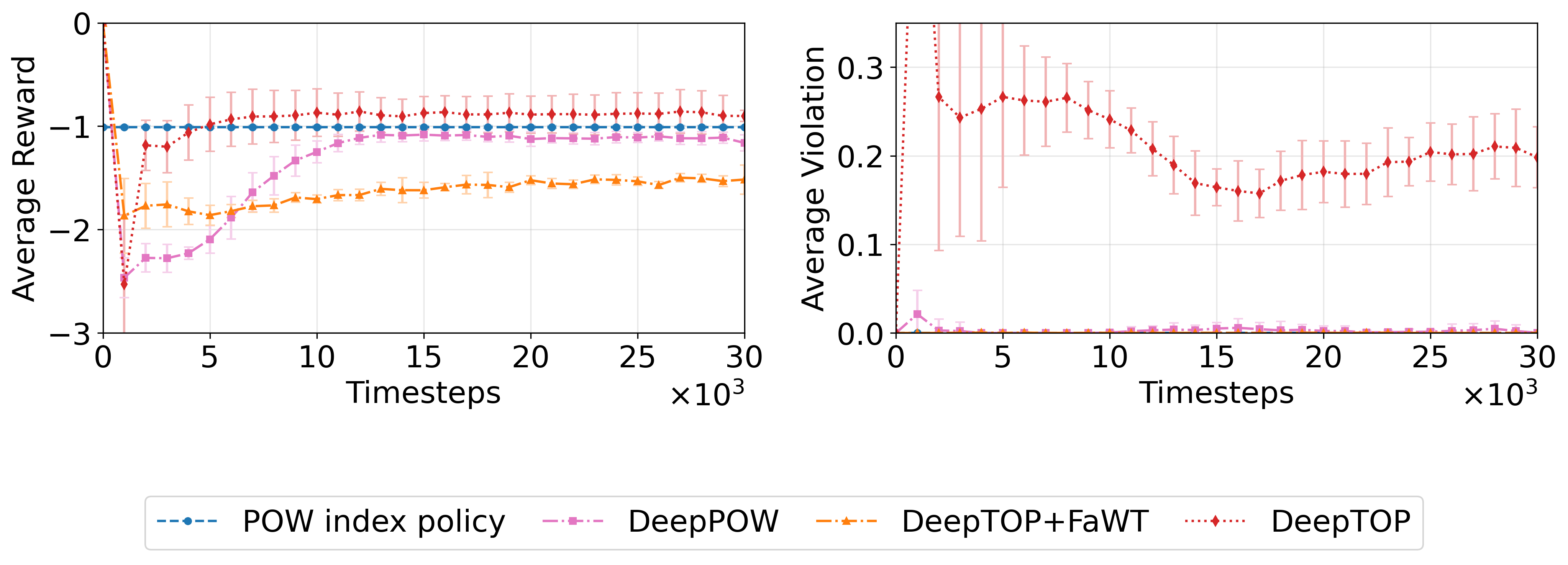}
    }

    \subfigure[$N=8$ and $C=1$]{
        \includegraphics[width=1\linewidth]{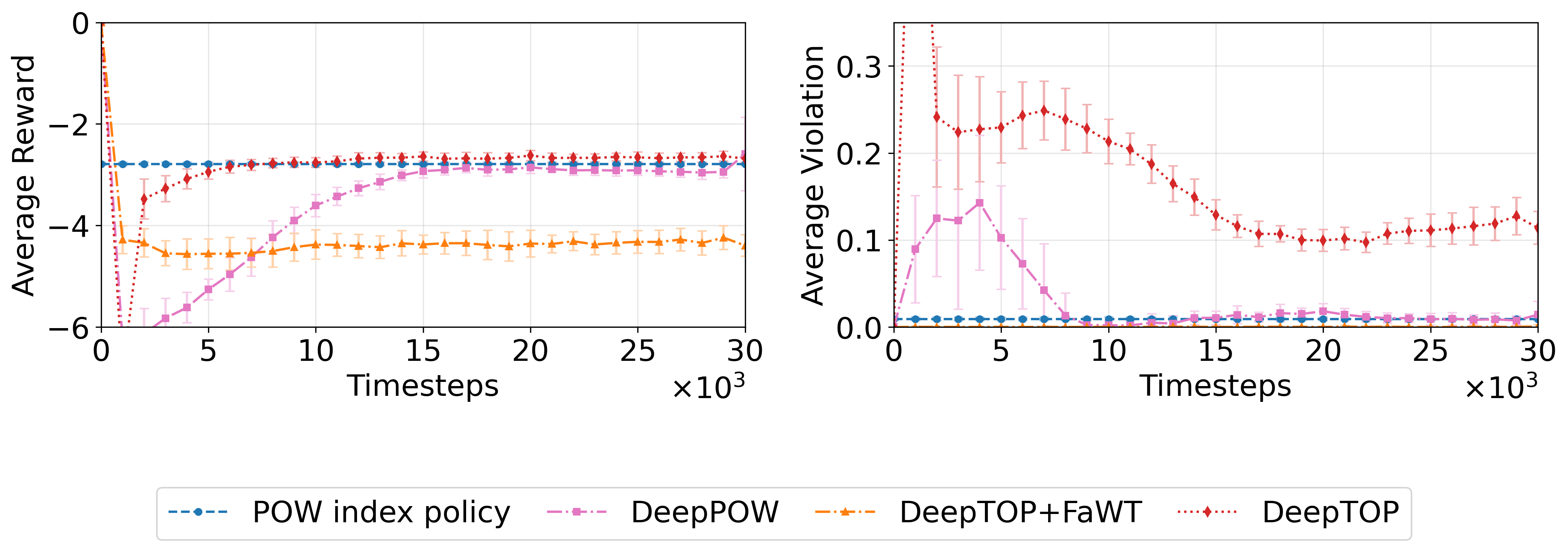}
    }
    \caption{Comparison of learning policies for remote sensing.}
    \label{fig:RS_deep}
\end{figure}

We consider a remote sensing scenario where a base station (BS) is gathering real-time surveillance data from a set of IoT devices. The performance of an IoT device is measured by its AoI, which is defined as the number of time steps since the last packet delivery. We consider unreliable and heterogeneous wireless channels, where every transmission for IoT device $n$ is successful with probability $p_n$. Further, since IoT devices are battery-powered, we consider that each IoT device expends one unit of energy for each transmission and has a stringent power consumption requirement. The goal of the BS is to find a scheduling policy that minimizes long-term average total AoI in the system while satisfying the power consumption requirement of each IoT device. 

We now discuss how to model each IoT device as a restless arm with penalty constraint. We model the state of device $n$ at time $t$, denoted by $s_{n,t}$, by its AoI at time $t$. The AoI of a device will increase by 1 in each time step without a successful transmission, and will drop to 1 whenever there is a successful transmission. To ensure a finite state space, we cap the AoI at 30. In this application, $a_{n,t}$ corresponds to the indicator function that the BS schedules device $n$ for transmission at time $t$. Since each transmission from device $n$ is successful with probability $p_n$, the transition kernel is defined as: $ P_n(s_{n,t+1}=\min\{s+1,30\}|s_{n,t}=s,a_{n,t}=0)=1$,
$P_n(s_{n,t+1}=\min\{s+1,30\}|s_{n,t}=s,a_{n,t}=1)=1-p_n$ and  $P_n(s_{n,t+1}=1|s_{n,t}=s,a_{n,t}=1)=p_n$. The reward is defined as the negative normalized AoI: $R_n(s_{n,t},a_{n,t}) = -\frac{s_{n,t}}{30}$, and the penalty function is defined as: $g_n(s_{n,t},a_{n,t}) = a_{n,t}$, where $B_n$ specifies the energy budget. Parameter settings are shown in Appendix~\ref{sec:impl}.



Fig.~\ref{fig:RS_reward} shows the average reward and average constraint satisfaction for the remote sensing problem across the two system settings. The POW index policy remains consistently close to the fluid-relaxed benchmark. Among all feasible policies, POW achieves the lowest AoI while satisfying per-device energy constraints. Comparing baseline policies, the Whittle index policy achieves lower AoI, but this comes at the cost of violating energy constraints, indicating that unconstrained scheduling is infeasible in practice. Both DPP and FaWT have considerably worse AoI than POW. Fig.~\ref{fig:RS_deep} presents the learning performance under unknown transition dynamics, reward and constraint violation results are averaged over the last 1,000 time steps. DeepPOW quickly converges to the reward level of POW while maintaining constraint satisfaction. DeepPOW also significantly outperforms DeepTOP and DeepTOP+FaWT.

\subsection{Throughput Maximization with Service Regularity}

\begin{figure}[t]
    \centering
    \subfigure[Comparison of index policies with $N=4$ and $C=1$]{
        \includegraphics[width=1\linewidth]{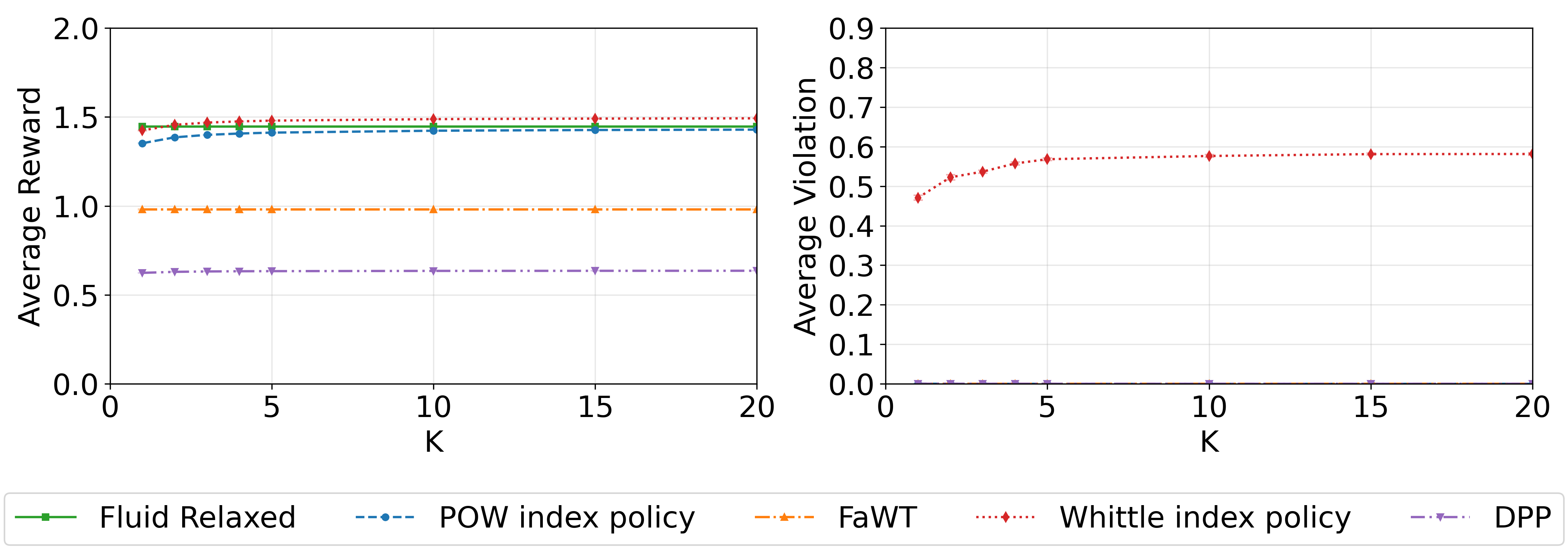}
    }

    \subfigure[Comparison of index policies with $N=6$ and $C=1$]{
        \includegraphics[width=1\linewidth]{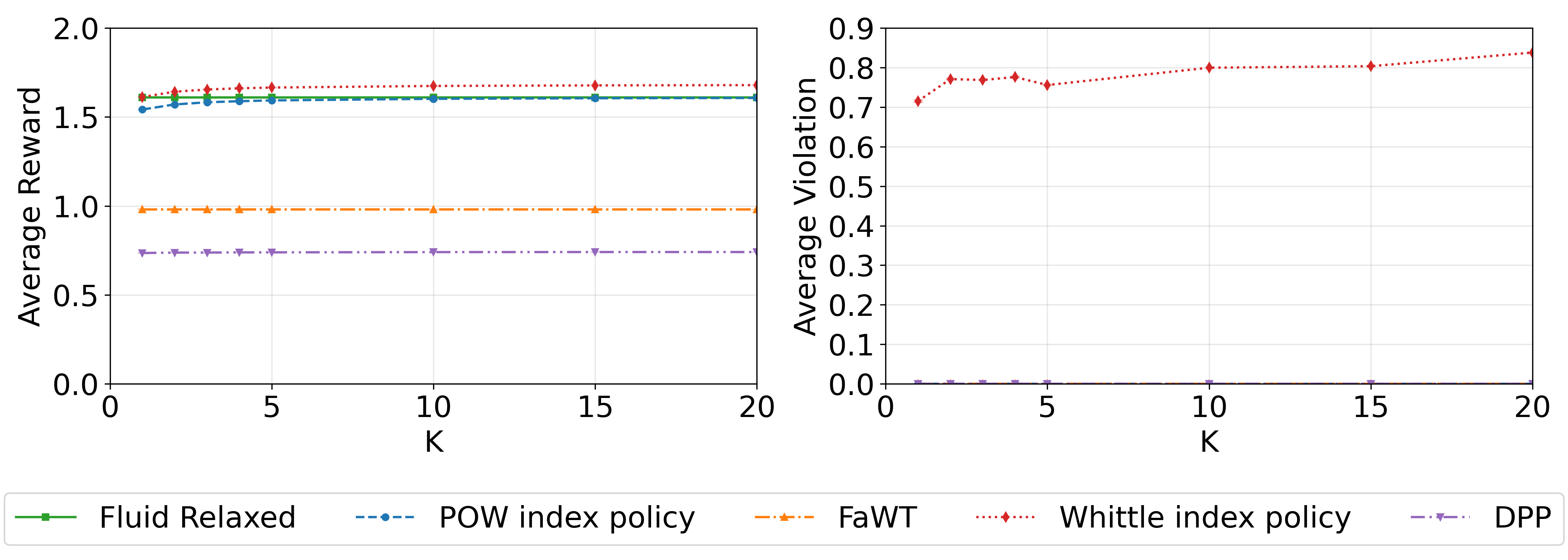}
    }
    \caption{Average reward and average constraint violation for throughput maximization with service regularity.}
    \label{fig:TP_regularity}
\end{figure}

\begin{figure}[t]
    \centering
    \subfigure[$N=4$ and $C=1$]{
        \includegraphics[width=1\linewidth]{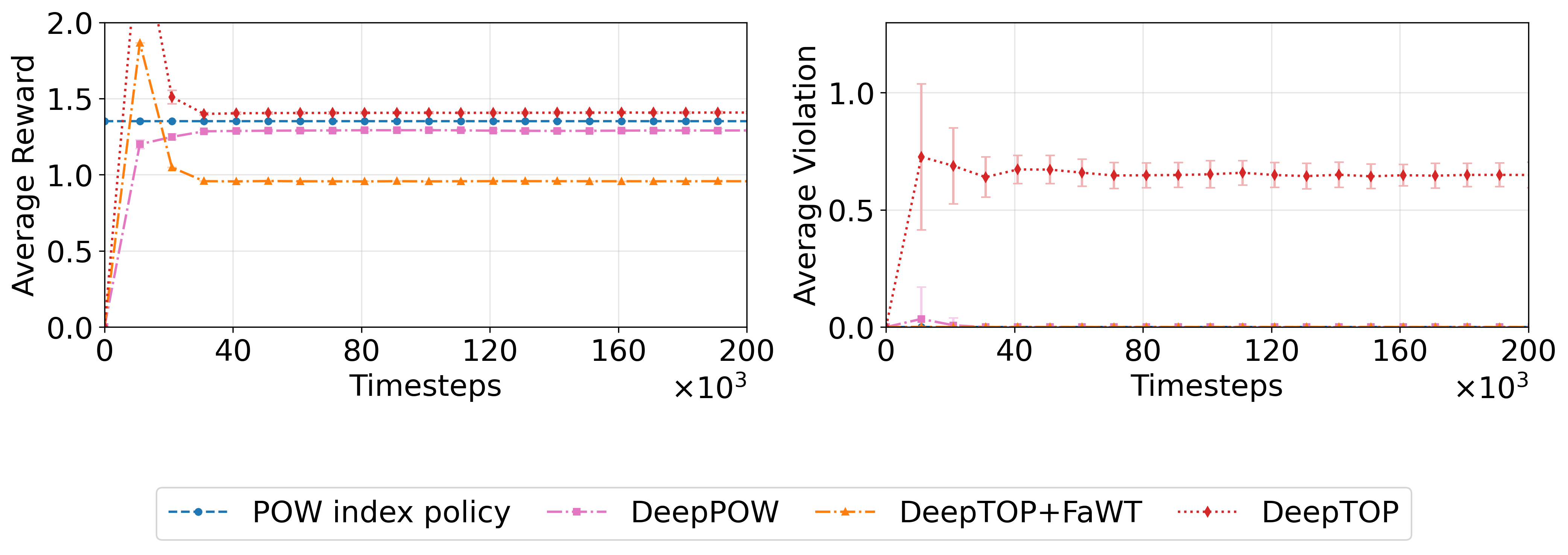}
    }

    \subfigure[$N=6$ and $C=1$]{
        \includegraphics[width=1\linewidth]{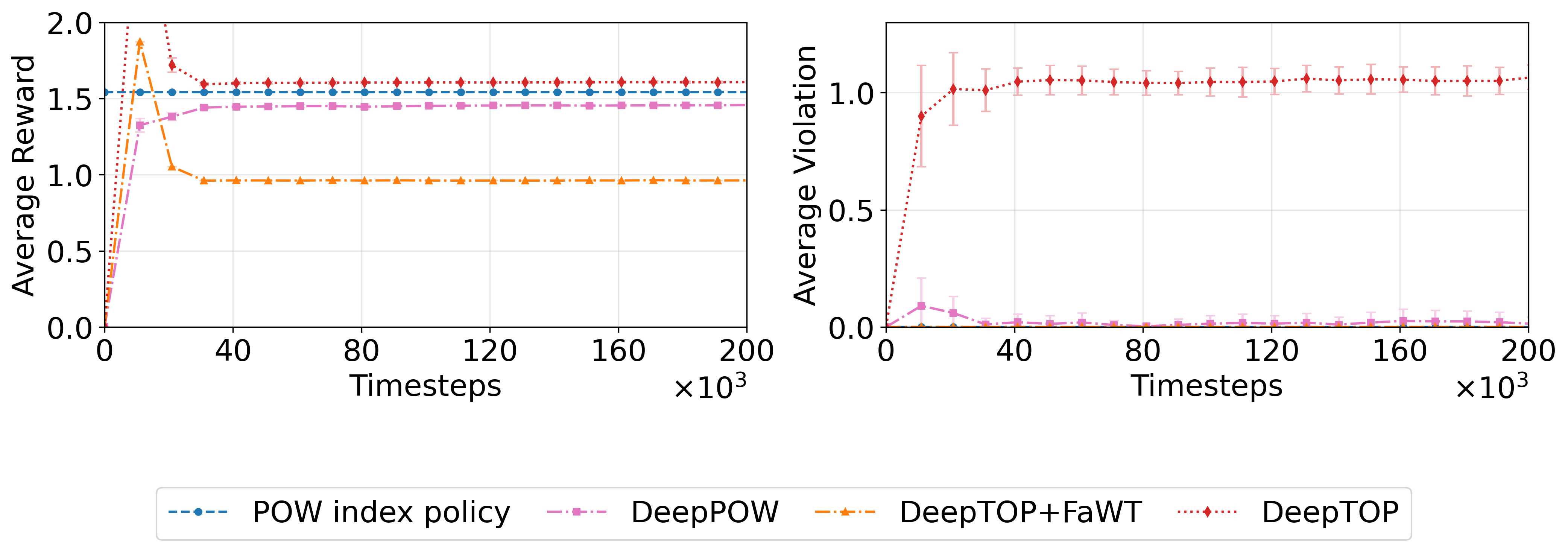}
    }
    \caption{Comparison of learning policies for throughput maximization with service regularity.}
    \label{fig:TP_regularity_deep}
\end{figure}

In the last application, we consider a throughput maximization problem with service regularity constraints, where the BS must maximize long-term total throughput while ensuring that each UE is scheduled frequently enough to maintain fresh status updates at the BS. Service regularity depends on how long a UE remains unscheduled, making the scheduling problem sensitive not only to throughput opportunities but also to temporal fairness. Consequently, the policy must balance instantaneous throughput gains against the risk of excessive update delay.

We model each UE as a restless arm whose state contains both channel state and service regularity information. We let $d_{n,t}\in\{1, 2, \dots, 10\}$ be the channel state of UE $n$ at time $t$. If UE $n$ is scheduled when its channel state is $d_n$, then it can deliver $(\theta_{n,0}+\theta_{n,1}d_{n,t})$ packets. We use $h_{n,t}\in\{0, 1,\dots,9\}$ to indicate the number of slots since the last time UE $n$ was scheduled. The state of UE $n$ at time $t$ is then denoted by $s_{n,t}=(d_{n,t},h_{n,t})$. Since the goal is to maximize total throughput while ensuring service regularity, we define the reward function as $R_n(s_{n,t},a_{n,t}) = (\theta_{n,0}+\theta_{n,1}d_{n,t})a_{n,t}$ and the penalty function as $g_n(s_{n,t},a_{n,t}) = \frac{h_{n,t}}{9}$. The parameter choices for $\vec{\theta}_0$, $\vec{\theta}_1$, and $\vec{B}$ are summarized in Appendix~\ref{sec:impl}.

FaWT does not directly apply to this environment because the constraint is defined through service regularity rather than a minimum activation ratio. To adapt FaWT, we first allocate resources to users with positive regularity violations, then use the remaining budget to serve users with high Whittle indices. 

Fig.~\ref{fig:TP_regularity} illustrates the average reward and constraint violation under varying network scales, while Fig.~\ref{fig:TP_regularity_deep} shows the learning performance over time. Like the previous two applications, the POW index policy is very close to the Fluid-Relaxed solution and outperforms other policies in all scenarios. The DeepPOW policy demonstrates strong learning performance, closely tracking the POW policy in both reward and feasibility. 
\section{Conclusion}\label{sec:conclusion}


We presented a general framework for Restless Multi-Armed Bandits with individual penalty constraints, capturing per-user service requirements such as energy limits, activation frequencies, or data freshness. We proposed the Penalty-Optimal Whittle (POW) index policy. We proved that the POW index policy is asymptotic optimal under fluid scaling. To handle unknown dynamics, we developed DeepPOW, a reinforcement learning approach that learns the POW indices from interaction. This work provides a scalable approach for constrained resource allocation and opens avenues for future research, such as multi-resource constraints, average-reward settings, and learning in high-dimensional or partially observed systems.
\section{Acknowledgement}
This material is based upon work supported in part by NSF under Award Number CCF-2332800, U.S. Army Combat Capabilities Development Command under Grant Number W911NF2520046, and the U.S. Office of Naval Research under Grant Number N000142412615. Portions of this research were conducted with the advanced computing resources provided by Texas A\&M High Performance Research Computing.
\bibliographystyle{ACM-Reference-Format}
\bibliography{references}

\newpage
\appendix

\section{Proof of Theorem~\ref{theorem:fluidopt}}\label{sec:A}

We prove that $(\vec{x}^{\text{Ind}}, \lambda^{\text{Ind}}, [\mu_n^*(\lambda^{\text{Ind}})])$ satisfies the KKT conditions of the \textbf{Fluid-Relaxed($\vec{y}^{\text{Rel}}$)} problem. Since \textbf{Fluid-Relaxed} is a linear program, KKT conditions are both necessary and sufficient for optimality, so this establishes $(\vec{x}^{\text{Ind}}, \lambda^{\text{Ind}}, [\mu_n^*(\lambda^{\text{Ind}})]) = (\vec{x}^{\text{Rel}}, \lambda^{\text{Rel}}, [\mu_n^{\text{Rel}}])$ and hence $\pi^{\text{Rel}}(\vec{y}^{\text{Ind}}) = \pi^{\text{Ind}}(\vec{y}^{\text{Ind}})$.

The KKT conditions of \textbf{Fluid-Relaxed($\vec{y}^{\text{Rel}}$)} are:
\begin{enumerate}
    \item \textit{(Primal feasibility)} $\sum_s \sum_a g_n(s,a)x_{n,s,a} \leq \frac{B_n}{1-\beta}$ for all $n$, and $\sum_{n,s} x_{n,s,1} = \frac{C}{1-\beta}$.\label{item:Primal}
    \item \textit{(Dual feasibility)} $\mu^{\text{Rel}}_n\geq0$ for all $n$, and $\lambda^{\text{Rel}}\geq0$.\label{item:Dual}
    \item \textit{(Complementary slackness)}
    \begin{align*}
      \mu^{\text{Rel}}_n\!\left(\frac{B_n}{1-\beta}-\sum_s\sum_a g_n(s,a)x^{\text{Rel}}_{n,s,a}\right) &= 0, \; \forall n,\\
      \lambda^{\text{Rel}}\!\left(\frac{C}{1-\beta}-\sum_n\sum_s x^{\text{Rel}}_{n,s,1}\right) &= 0.
    \end{align*}\label{item:Complementary}
    \item \textit{(Lagrangian optimality)} $\vec{x}^{\text{Rel}}$ maximizes, for each $n$,
    \begin{align*}
      &\sum_s\sum_a\!\left(r_n(s,a)+\mu^{\text{Rel}}_n\!\left(\tfrac{B_n}{1-\beta}-g_n(s,a)\right)-\lambda^{\text{Rel}}\mathbf{1}_{\{a=1\}}\right)x_{n,s,a}\\
      \text{s.t.}\quad &\sum_a x_{n,s,a}-\beta\!\sum_{s'}\sum_{a'}P_n(s|s',a')x_{n,s',a'}=y^{\text{Rel}}_{n,s},\;\forall s.
    \end{align*}\label{item:lag}
\end{enumerate}

We now define the quantities associated with $\pi^{\text{Ind}}$. Let $\vec{z}^{\text{Ind}} = \pi^{\text{Ind}}(\vec{y}^{\text{Ind}})$ be the per-step allocation at the global attractor. Since $\vec{y}^{\text{Ind}}$ is the steady-state distribution, we have $\vec{y}_t = \vec{y}^{\text{Ind}}$ and $\vec{z}_t = \vec{z}^{\text{Ind}}$ for all $t \geq 0$ when $\vec{y}_0 = \vec{y}^{\text{Ind}}$. Hence $\vec{x}^{\text{Ind}} := \sum_{t=0}^\infty \beta^t \vec{z}^{\text{Ind}} = \vec{z}^{\text{Ind}}/(1-\beta)$. Let $\lambda^{\text{Ind}}$ be the Lagrange multiplier of the capacity constraint in \textbf{Fluid-Ind($\vec{y}^{\text{Ind}}$)} (Eq.~\ref{equation:fluid-index1}), and let $\mu_n^{\text{Ind}} := \mu_n^*(\lambda^{\text{Ind}})$ be the optimal solution of the \textbf{Dual-Arm-n} problem (Eq.~\ref{Arm_n_dual}) for that $\lambda^{\text{Ind}}$. We assume $\mu_n^{\text{Ind}}$ is finite throughout.\rev{ For each arm $n$, the primal problem corresponding to \textbf{Dual-Arm-n} admits an equivalent finite-dimensional occupation-measure linear programming formulation. Under the feasibility and finiteness assumptions of Theorem 1, strong duality holds at $\lambda=\lambda^{\text{Ind}}$. Since $\mu_n^{\text{Ind}}=\mu_n^*(\lambda^{\text{Ind}})$, the value $L_n(\lambda^{\text{Ind}},\mu_n^{\text{Ind}})$ equals the optimal primal value, and optimal primal-dual pairs satisfy complementary slackness.}

\textbf{Condition 1 (Primal feasibility).}
We verify both primal constraints for $\vec{x}^{\text{Ind}}$.

\textit{Capacity constraint:} Since $\lambda^{\text{Ind}}$ is a strict index separator (Definition~\ref{D2}), we have $\sum_{n,s} z^{\text{Ind}}_{n,s,1} = C$. Using $\vec{x}^{\text{Ind}} = \vec{z}^{\text{Ind}}/(1-\beta)$, this gives $\sum_{n,s} x^{\text{Ind}}_{n,s,1} = C/(1-\beta)$.

\textit{Penalty constraint:} Since $\mu_n^{\text{Ind}} = \mu_n^*(\lambda^{\text{Ind}})$ is the optimal dual variable of the penalty constraint in \textbf{Dual-Arm-n}, the primal constraint of that problem requires $\sum_s \sum_a g_n(s,a) x^{\text{Ind}}_{n,s,a} \leq B_n/(1-\beta)$, which is exactly the penalty constraint of \textbf{Fluid-Relaxed}.

\textbf{Condition 2 (Dual feasibility).}
By Definition~\ref{D2}, the strict index separator $\lambda^{\text{Ind}}$ satisfies $\lambda^{\text{Ind}} > 0$. The constraint $\mu_n \geq 0$ in the \textbf{Dual-Arm-n} problem (Eq.~\ref{Arm_n_dual}) ensures $\mu_n^{\text{Ind}} \geq 0$.

\textbf{Condition 3 (Complementary slackness).}
\textit{For $\lambda^{\text{Ind}}$:} Since $\lambda^{\text{Ind}} > 0$, we must show $\sum_{n,s} x^{\text{Ind}}_{n,s,1} = C/(1-\beta)$. This was established in Condition~1.

\textit{For $\mu_n^{\text{Ind}}$:} We must show
\[
\mu_n^{\text{Ind}}\!\left(\frac{B_n}{1-\beta} - \sum_s\sum_a g_n(s,a) x^{\text{Ind}}_{n,s,a}\right) = 0.
\]
We consider three cases. If $\sum_s\sum_a g_n(s,a) x^{\text{Ind}}_{n,s,a} < B_n/(1-\beta)$, then minimizing over $\mu_n \geq 0$ in \textbf{Dual-Arm-n} yields $\mu_n^{\text{Ind}} = 0$, so the product is zero. If $\sum_s\sum_a g_n(s,a) x^{\text{Ind}}_{n,s,a} = B_n/(1-\beta)$, the product is zero for any $\mu_n^{\text{Ind}} \geq 0$. If $\sum_s\sum_a g_n(s,a) x^{\text{Ind}}_{n,s,a} > B_n/(1-\beta)$, the minimizer of \textbf{Dual-Arm-n} would be $\mu_n^{\text{Ind}} = \infty$, contradicting our assumption of finiteness. Hence this case cannot occur, and complementary slackness holds.

\textbf{Condition 4 (Lagrangian optimality).}
We must show that $\vec{x}^{\text{Ind}}$ maximizes the Lagrangian in condition~(\ref{item:lag}) under $\lambda^{\text{Ind}}$ and $[\mu_n^{\text{Ind}}]$.

The $N$ subproblems of that Lagrangian decompose independently over arms. For group $n$, the subproblem is:
\begin{align}\label{problem:lagOpt_fluid-relaxed}
    \max_{x_{n,s,a}}\;&\sum_s\sum_a\!\left(r_n(s,a)+\mu^{\text{Ind}}_n\!\left(\tfrac{B_n}{1-\beta}-g_n(s,a)\right)\right.\nonumber\\
    &\left.\quad-\lambda^{\text{Ind}}\mathbf{1}_{\{a=1\}}\right) x_{n,s,a} \nonumber \\
     \text{s.t.}\;&\sum_a x_{n,s,a}-\beta\!\sum_{s'}\sum_{a'}P_n(s|s',a')x_{n,s',a'}\nonumber\\
     &\quad =y^{\text{Rel}}_{n,s}, \quad \forall s.
\end{align}
This is equivalent to the \textbf{Arm-n} MDP (Eq.~\ref{Arm_n_lagrangian}) with $\mu_n = \mu_n^{\text{Ind}}$ and $\lambda = \lambda^{\text{Ind}}$. The optimal policy is therefore $\pi_n^*(s_n, \lambda^{\text{Ind}}, \mu_n^{\text{Ind}})$.

Since $\pi^{\text{Ind}}$ is at steady state, $z^{\text{Ind}}_{n,s,a,t} = z^{\text{Ind}}_{n,s,a}$ for all $t$, and
\[
x^{\text{Ind}}_{n,s,a} = \sum_{t=0}^{\infty}\beta^t z^{\text{Ind}}_{n,s,a,t} = \frac{z^{\text{Ind}}_{n,s,a}}{1-\beta}.
\]
By the global attractor property, starting from $\vec{y}_0 = \vec{y}^{\text{Ind}}$ gives $\vec{y}^{\text{Rel}} = \vec{y}^{\text{Ind}}$. The optimal solution of problem~(\ref{problem:lagOpt_fluid-relaxed}) is therefore:
\[
  x^{\text{Ind}}_{n,s,1} =
\begin{cases}
\dfrac{y^{\text{Ind}}_{n,s}}{1-\beta}, & \text{if }\pi_n^*(s_n, \lambda^{\text{Ind}}, \mu_n^{\text{Ind}})=1,\\
0, & \text{otherwise,}
\end{cases}
\]
\[
  x^{\text{Ind}}_{n,s,0} =
\begin{cases}
0, & \text{if }\pi_n^*(s_n, \lambda^{\text{Ind}}, \mu_n^{\text{Ind}})=1, \\
\dfrac{y^{\text{Ind}}_{n,s}}{1-\beta}, & \text{otherwise,}
\end{cases}
\]
which matches $\vec{x}^{\text{Ind}} = \vec{z}^{\text{Ind}}/(1-\beta)$. Thus $\vec{x}^{\text{Ind}}$ satisfies condition~(\ref{item:lag}).

Since all four KKT conditions are satisfied, $(\vec{x}^{\text{Ind}}, \lambda^{\text{Ind}}, [\mu_n^*(\lambda^{\text{Ind}})])$ is the unique optimal solution of \textbf{Fluid-Relaxed($\vec{y}^{\text{Rel}}$)}, i.e., it equals $(\vec{x}^{\text{Rel}}, \lambda^{\text{Rel}}, [\mu_n^{\text{Rel}}])$. Consequently, $\pi^{\text{Rel}}(\vec{y}^{\text{Ind}}) = \pi^{\text{Ind}}(\vec{y}^{\text{Ind}})$ and $\vec{y}^{\text{Rel}} = \vec{y}^{\text{Ind}}$, completing the proof. \qed

\section{Proof of Theorem~\ref{thm:policygrad}}\label{sec:C}

We drop the subscript $n$ for simplicity. Taking the gradient of Eq.~(\ref{eq:jFunction}):
\begin{align} \label{eq:gradientJa}
\nabla_{\phi}J^{\phi} = \nabla_{\phi} \sum_{s \in S} \int_{-M}^{+M} Q^{\phi}(s,a,\lambda,\mu^{\zeta}(\lambda))\,d\lambda.
\end{align}

Renumber states so that $w^{\phi}(s_{|S|}) > \cdots > w^{\phi}(s_1)$. Define $\mathcal{M}^0 = -M$, $\mathcal{M}^k = w^{\phi}(s_k)$ for $1 \leq k \leq |S|$, and $\mathcal{M}^{|S|+1} = +M$. Let $S_k = \{s_k, \ldots, s_{|S|}\}$. For $\lambda \in (\mathcal{M}^k, \mathcal{M}^{k+1})$, the policy $\pi^\phi$ activates $s$ if and only if $s \in S_{k+1}$. Define $\hat{\pi}^{k+1}$ as the policy that activates $s$ iff $s \in S_{k+1}$, and let $\hat{Q}^{k+1}(s,a,\lambda)$ denote its Q-function at $(\lambda,\mu^\zeta(\lambda))$. Since $Q^\phi(s,a,\lambda,\mu^\zeta(\lambda)) = \hat{Q}^{k+1}(s,a,\lambda)$ for all $\lambda \in (\mathcal{M}^k, \mathcal{M}^{k+1})$, we can rewrite Eq.~(\ref{eq:gradientJa}) as:
\begin{align}\label{eq:gradientJb}
\nabla_{\phi}J^{\phi} = \sum_{k=0}^{|S|}\sum_{s \in S} \nabla_{\phi} \int_{\mathcal{M}^k}^{\mathcal{M}^{k+1}} \hat{Q}^{k+1}(s,\hat{\pi}^{k+1}(s),\lambda)\,d\lambda.
\end{align}

Applying the Leibniz integral rule to each term in Eq.~(\ref{eq:gradientJb}):

\begin{align} \label{eq:gradientJe}
\nabla_{\phi}J^{\phi} = \sum_{k=0}^{|S|}\sum_{s \in S}
\Big[&\hat{Q}^{k+1}(s,\hat{\pi}^{k+1}(s),\mathcal{M}^{k+1})\nabla_{\phi}\mathcal{M}^{k+1}
\nonumber\\
-\,&\hat{Q}^{k+1}(s,\hat{\pi}^{k}(s),\mathcal{M}^{k})\nabla_{\phi}\mathcal{M}^{k}\Big],
\end{align}

where the integral term $\int \nabla_\phi \hat{Q}^{k+1}(\cdot), d\lambda$ vanishes since it is the state-action value function under the fixed policy $\hat{\pi}^{k+1}$, i.e., $\hat{Q}^{k+1}(s,\hat{\pi}^{k+1}(s),\lambda)$, and thus does not depend on $\phi$ for $\lambda \in (\mathcal{M}^k, \mathcal{M}^{k+1})$.

Since $\hat{\pi}^{k+1}(s) = \hat{\pi}^k(s)$ for all $s \neq s_k$, the boundary terms in Eq.~(\ref{eq:gradientJe}) telescope. Re-indexing over $k$ from $1$ to $|S|$ and isolating the $s_k$ contribution:
\begin{align}
\nabla_{\phi}J^{\phi}
&= \sum_{k=1}^{|S|}\Big[
  \hat{Q}^{k}(s_k,1,\mathcal{M}^{k})
  - \hat{Q}^{k}(s_k,0,\mathcal{M}^{k})
\Big]\nabla_{\phi}\mathcal{M}^{k} \nonumber\\
&= \sum_{s \in S}\Big[
  Q^{\phi}(s,1,w^{\phi}(s),\mu^{\zeta}(w^{\phi}(s))) \nonumber\\
&\quad - Q^{\phi}(s,0,w^{\phi}(s),\mu^{\zeta}(w^{\phi}(s)))
\Big]\nabla_{\phi}w^{\phi}(s),
\end{align}
which is Eq.~(\ref{eq:jGradient}). \qed

\section{Implementation Details}\label{sec:impl}

All neural networks (actor and critic) consist of two fully connected hidden layers with 128 neurons each and are trained using the Adam optimizer. \rev{Discounted formulations with discount factors close to one are commonly used in reinforcement learning for applications evaluated under a long-run average-reward criterion \cite{roch2025reduction}. Accordingly, we set $\beta=0.99$ in our experiments and evaluate performance using long-run average reward and constraint violation.}The dual variable range is set to $U=5$ for the throughput maximization with activation constraint and throughput maximization with service regularity scenarios, and $U=10$ for the remote sensing scenario. Each DeepPOW and DeepTOP training run uses 400{,}000 time steps. All results are averaged over 20 independent runs, each evaluated over 10{,}000 time steps.

\rev{For DeepTOP and DeepTOP+FaWT, we use the same neural network architectures and training settings as DeepPOW where applicable to ensure a fair comparison. DeepTOP+FaWT combines the indices learned by DeepTOP with the FaWT arm-selection rule.}

Parameter settings for all three scenarios are listed in Table~\ref{table:all_settings}.

\begin{table}[H]
\centering
\footnotesize
\caption{Parameter settings for all three simulation scenarios.}
\label{table:all_settings}
\begin{tabular}{c|ccc}
\multicolumn{4}{c}{\textbf{TP w/ Activation Constraint}} \\ \hline
$(N,C)$ & $\vec{\theta}_0$ & $\vec{\theta}_1$ & $\vec{\delta}$ \\ \hline
$(4,1)$ & [.3,.5,.7,.9] & $4{\times}[.02]$ & [.35,.35,.05,.05] \\
$(6,1)$ & [.3,.4,\ldots,.9] & $6{\times}[.02]$ & [.2,.2,.1,.1,.05,.05] \\ \hline
\multicolumn{4}{c}{\textbf{Remote Sensing}} \\ \hline
$(N,C)$ & $\vec{p}$ & \multicolumn{2}{c}{$\vec{B}$} \\ \hline
$(4,1)$ & [.2,.4,.6,.8] & \multicolumn{2}{c}{[.9,.9,.1,.1]} \\
$(8,1)$ & [.2,.3,\ldots,.9] & \multicolumn{2}{c}{[.4,.4,.2,.2,.1,.1,.05,.05]} \\ \hline
\multicolumn{4}{c}{\textbf{TP w/ Service Regularity}} \\ \hline
$(N,C)$ & $\vec{\theta}_0$ & $\vec{\theta}_1$ & $\vec{B}$ \\ \hline
$(4,1)$ & [.4,.5,.6,.7] & $4{\times}[.1]$ & [.2,.2,.1,.1] \\
$(6,1)$ & [.4,.5,\ldots,.9] & $6{\times}[.1]$ & [.4,.4,.8,.8,1,1]
\end{tabular}
\end{table}

\end{document}